\documentclass{article}

\usepackage{arxiv}

\usepackage{hyperref}       %
\usepackage{times}  
\usepackage{helvet}  
\usepackage{courier}  
\usepackage{graphicx} 
\usepackage{natbib}  
\usepackage{caption} 
\usepackage{newfloat}
\usepackage{listings}

\usepackage{microtype}
\usepackage{subfigure}
\usepackage{booktabs} 
\usepackage{amsfonts}       
\usepackage{amssymb}
\usepackage{amsmath}
\usepackage{amsthm}
\usepackage[noend]{algpseudocode}
\usepackage{algorithmicx,algorithm}
\usepackage{cancel}
\usepackage{appendix}

\newtheorem{thm}{Theorem}

\newtheorem{lemma}{Lemma}
\usepackage{multirow}
\usepackage[american]{babel}

\title{Unified Policy Optimization \\for Continuous-action Reinforcement Learning \\in Non-stationary Tasks and Games}

\usepackage{authblk}
\newcommand\nnfootnote[1]{%
  \begin{NoHyper}
  \renewcommand\thefootnote{}\footnote{#1}%
  \addtocounter{footnote}{-1}%
  \end{NoHyper}
}
\author[1,2]{\textbf{Rong-Jun Qin}}
\author[1,2]{\textbf{Fan-Ming Luo}}
\author[3]{\textbf{Hong Qian}}
\author[1,2]{\textbf{Yang Yu}\textsuperscript{\dag}}
\affil[1]{National Key Laboratory for Novel Software Technology, Nanjing University, China}
\affil[2]{Polixir.ai}
\affil[3]{School of Computer Science and Technology, East China Normal University, China}
\affil[ ]{\texttt{\{qinrj,luofm\}@lamda.nju.edu.cn}, \texttt{hqian@cs.ecnu.edu.cn}, \texttt{yuy@nju.edu.cn}}

\renewcommand{\headeright}{}
\renewcommand{\shorttitle}{}

\date{}
\begin{document}
\maketitle

\begin{abstract}
This paper addresses policy learning in non-stationary environments and games with continuous actions. Rather than the classical reward maximization mechanism, inspired by the ideas of follow-the-regularized-leader (FTRL) and mirror descent (MD) update, we propose a no-regret style reinforcement learning algorithm PORL for continuous action tasks. We prove that PORL has a \emph{last-iterate} convergence guarantee, which is important for adversarial and cooperative games. Empirical studies show that, in stationary environments such as MuJoCo locomotion controlling tasks, PORL performs equally well as, if not better than, the soft actor-critic (SAC) algorithm; in non-stationary environments including dynamical environments, adversarial training, and competitive games, PORL is superior to SAC in both a better final policy performance and a more stable training process.
\end{abstract}

\nnfootnote{\dag: Corresponding author.}

\section{Introduction}
Reinforcement learning (RL) \citep{RLbook:2018} as a policy optimization tool has endowed many systems with a superhuman performance \citep{atari_dqn:15}. These environments for training an RL agent are often stationary, i.e., the transition and reward function is fixed. Thus, the objective of policy optimization is simply to maximize the episodic reward. However, the transition function of learning environments may constantly change over time. Such non-stationary environments include dynamical scenarios from the real world, adversarial training scenarios (e.g., GAN~\citep{gan:14}, GAIL~\citep{gail:16}), and games. The nonstationarity in adversarial training and games is not induced by the change of the environment itself but is mainly caused by the changing policies of adversaries and opponents (or teammates in the cooperative multi-agent systems). Consequently, RL algorithms that are tailored to stationary environments can easily fluctuate or even diverge in non-stationary environments. From some practical observations, the learning process of RL algorithms may also fluctuate in stationary environments, possibly due to the multiple near-optimal solutions. Maximum entropy is a popular technique to guarantee that the \textit{soft} optimal solution is unique and close to the original solution. One of the representative maximum entropy algorithms SAC \citep{SAC}, is shown to be very stable across many continuous control tasks on MuJoCo \citep{MuJoCo}. Nevertheless, a reward maximization algorithm like SAC does not necessarily converge in non-stationary environments or adversarial training in theory. In these domains, techniques for stabilizing the training and guaranteeing the convergence are expected, such as averaging the historical policies as done in NFSP \citep{nfsp_heinrich:16}. Besides, population-based techniques from empirical game theory \citep{egta:06} have been employed in training GAN \citep{fictitiousGAN:18,DO_GAN:22} to achieve a better convergence property.

On the contrary, online learning provides techniques for learning against non-stationary distributions, which are sourced from the no-regret learning (also Hannan consistency) principle \citep{online_and_games:06}. Intuitively, the no-regret principle minimizes the cumulative regret iteratively, and the learning algorithm will remedy current policy based on past plays or historical policies and thus can adapt to the non-stationary environment. Online learning often assumes the decision space is finite, e.g., the policy is a probability distribution over the discrete action space, which is a probability vector. While in continuous action space, the policy is a probability density function, thus the solution is a functional which can be viewed as an infinite dimensional vector. A fundamental change is that the convex property with respect to the action, which is a key to the proof of the no-regret learning, may not generally hold with functionals. Besides, we noticed that the convergence of online learning is often established for the \textit{average} policy, while the average of historical policies is non-trivial for neural network policies, thus we resort to the last-iterate convergence. In the experiments, we did observe the learning of average policy in NFSP can be less efficient.

This work is then motivated by the celebrated no-regret learning framework, follow-the-regularized-leader (FTRL). Employing the FTRL framework with the entropy regularizer on the regularized MDP, we derive a policy update rule tailored for the continuous action space. It can also be applied in discrete action space seamlessly. In continuous action space, the proposed method has the last-iterate convergence that converges to the \textit{soft} optimal policy in single-player tasks and the quantal response equilibrium \citep{QRE:1995} in two-player tasks, under some mild assumptions. The experiments on stationary MuJoCo environments with continuous action space are competitive with the state-of-the-art model-free algorithm, SAC. While in the non-stationary environments and adversarial training, the experimental results suggest the proposed method can achieve more robust performance and a stable learning process. Finally, in two-player competitive games, the proposed method achieves higher expected returns when pitting against SAC and NFSP, and is more robust against the exploiter.

\section{Preliminaries}
Consider a stochastic game (Markov game) $G$ with a set $N$ of $n$ players with finite horizon $H$. $n\in \{1,2\}$ in this work and it reduces to single-agent MDP when $n=1$. Each player $i$ has a state space $S_i$, an action space $A_i=A(S_i)$ and a reward function $r_i:S\times A_1\times\cdots\times A_n\rightarrow \mathbb{R}$, which is in the range $[-R_{\max},R_{\max}]$. If the action space is discrete, we assume that the cardinality of the action space for each player $i$ is the same, i.e., $|A_i|=d$. When the action space is continuous, it is assumed to be compact. The transition function $P(s^\prime|s,\mathbf{a})$ transits from $s$ to $s^\prime$ after all the agents take the joint action $\mathbf{a}$ (also $a$ for simplicity). We denote $\mathbf{\pi}=(\mathbf{\pi}_1,\cdots, \mathbf{\pi}_n)$ as the joint policy. Let $-i$ denote the players except $i$. For player $i$, the state value function is defined as $V_i^\pi(s)=\mathbf{E}_{(a_i,a_{-i})\sim\pi}[\sum_{t=0} \gamma^t r_i(s_t,a_t,{a}_{-i})|s_t=s]$ and state-action value function defines as $Q_i^\pi(s,a_i)=\mathbf{E}_{(a_i,a_{-i})\sim\pi}[\sum_{t=0} \gamma^t r_i(s_t,a_{i,t}, {a}_{-i})|s_t=s,a_{i,t}=a_i]$, where $\gamma\in[0,1)$ is the discount factor. The index $i$ of the player will be omitted in either $V_i$ or $\pi_i$ when it is clear. With function approximations, the solution concept of a stochastic game often shifts to the $\epsilon$-approximate Nash equilibrium ($\epsilon$-NE). For some $\epsilon\ge 0$, a joint policy $\pi$ is $\epsilon$-NE, if $\forall i, \max_{\pi_i^\prime}V_i(\pi_i^\prime, \pi_{-i})-V_i(\pi)\le\epsilon$.

While in an online learning problem, the agent makes decisions based on the historical plays and the observed utility. The online learning algorithm minimizes the regret by repeatedly playing with the opponent, thus online learning is naturally formulated by repeated games. The general setting of online learning mainly tackles single-state case (akin to the normal-form games) with finite action space, thus the utility function which is akin to the $Q$ function for player $i$ is denoted as $u_i:A_1\times\cdots\times A_n\rightarrow \mathbb{R}$. The cumulative regret after $T$ rounds for the player $i$ is defined as
{\setlength\abovedisplayskip{0.1cm}
\setlength\belowdisplayskip{0.1cm}
\begin{equation}
    R_i(a_i) = \sum_{t=1}^T u_i(a_i,a_{-i,t})-u_i(a_{i,t},a_{-i,t}). \nonumber
\end{equation}
}
The cumulative regret (of action $a_i$) measures the difference between the cumulative utility of the \textit{constant action} (or policy) and actual plays.
If the regret of a learning algorithm satisfies $Pr[\frac{1}{T}\lim_{T\rightarrow\infty}\max_{a_i\in A_i} R_i(a_i)\le 0]=1$ no matter how the opponents play, then this algorithm is said to be \textit{no-regret} (or Hannan consistency \citep{hannan_consistency}). Conversely, a no-regret learning algorithm can approximate the optimal policy when pitting against a fixed opponent. Besides, if the average regret satisfies $\frac{R_i}{T}\le\epsilon$ in two-player zero-sum games for all players, then the empirical play of the players converges to a $2\epsilon$-NE. Note the empirical play is equivalent to an average policy as $\frac{1}{T}\sum_{t=1}^T a_{i,t}$, thus no-regret learning inherently establishes the convergence of the average policy.

\textbf{Other notations.} In a compact space $A$, let $\int_{\mathrm{d}\mathbf{a}}f(\mathbf{a})=\int_{\mathbf{a}\in A}f(\mathbf{a})\mathrm{d}\mathbf{\mu(a)}$ for short, where $\mu$ is the Lebesgue measure and $\int {\mathrm{d}\mathbf{a}}=\text{vol}(A)$ is the volume of the compact space $A$. We will use $\mathrm{d}a$ for simplicity if there is no ambiguity. The inner product of two functions $f,g$ is defined as $\langle f,g \rangle=\int_{\mathrm{d}a}f(a)g(a)$. If $f,g$ are vectors, then $\langle f,g \rangle$ is the inner product of vectors. In the following, $\Vert f \Vert_p=(\int_{\mathrm{d}a}\vert f(a)\vert^p)^\frac{1}{p}$ is the $L_p$ norm of a function $f$. By default, $\Vert f\Vert$ is the $L_2$ norm. The (differential) entropy of a density function $p$ is $\mathcal{H}(p)=-\int_{a}p(a)\log p(a)\mathrm{d}a$ and Kullback-Leibler divergence $KL(p,q)=\int_a p(a)\log\frac{p(a)}{q(a)}\mathrm{d}a$.

\section{Related Work}
\subsection{Online learning, optimization and mirror descent} 

The notion of no-regret can trace back to the seminal work \citep{hannan_consistency} for two-player games. Because of the close relationship between the no-regret property and Nash equilibrium (NE) in two-player zero-sum games or coarse correlated equilibrium in general-sum games, techniques from online learning (or no-regret learning) are also widely used as tools for learning in games. For example, regret matching (RM) \citep{rm.Hart_Mas:2000} is a simple no-regret learning algorithm with a $O(\sqrt{T})$ worst-case regret, which later becomes the most widely used regret minimization component to solve large extensive-form games like Texas Hold'em \citep{CFR:07,MCCFR:09}. The follow-the-leader framework \citep{FTL:05}, most notably the regularized leader, captures many of the no-regret algorithms. Hedge (also multiplicative weights update, MWU) can be instantiated with the entropy regularization under the FTRL framework. An interesting result of NeuRD \citep{NeuralRD} reveals Hedge can also be recovered from the perspective of replicator dynamics. A faster convergence than $O(\sqrt{T})$ can be obtained with the predictive or optimistic version of FTRL (OFTRL) \citep{fast_reg_in_games:15, faster_hedge:20} for the average policy. Besides, OFTRL can also prove to have the last-iterate convergence \citep{last_iterative_OMWU:19} with different techniques. A recent work \citep{rew_trans_21} investigates the FTRL dynamics and shows that the policy-independent reward can not lead to convergence for the last-iterate policy if the equilibrium is a fully mixed strategy. Thus, the authors transform the original reward function to a policy-dependent one by adding a regularization term and then prove the convergence for the last-iterate policy. 
Another line of work follows the online optimization \citep{intro_to_OCO:16}, especially for linear or convex optimization. For two-player zero-sum games, the perturbed formulation from the perspective of online optimization is
{\setlength\abovedisplayskip{0.2cm}
\setlength\belowdisplayskip{0.2cm}
\begin{equation}
\min_{x\in\mathcal{X}}\min_{y\in\mathcal{Y}} f(x,y)+\eta_1 h_1(x)-\eta_2 h_2(y), \label{oco}    
\end{equation}
}
where $h_1,h_2$ are usually strongly-convex functions. When $f$ satisfied the zero-sum property and $\eta_1=\eta_2=0$, then the solution of Eq.\eqref{oco} is a NE. If $f(x,\cdot)$ is convex and $f(\cdot,y)$ is concave, the above problem can be solved by gradient descent ascent (GDA) with adaptive stepsize. Mirror descent (MD) can be viewed as a general-form regularized variant of vanilla GDA to solve \eqref{oco}, where the policy update in iteration $t$ is regularized with a Bregman divergence term $B_{\phi}(z,z_t)=\phi(z)-\phi(z_t)-\langle \nabla\phi(z_t),z-z_t\rangle$ with $z_t=(x_t,y_t)$. MD can update with a fixed step size and often results in faster or last-iterate convergence \citep{EG_pu_omwu:21,last_iter_efg:21,last_iter_eg_ogda:22}. 

Online learning and online optimization share many similarities and sometimes are equivalent \citep{connect_ol_oco}. For instance, RM and RM$^+$ can result from running FTRL and online MD respectively \citep{pred_blackwell:21}. With the choice of the distance function in Bregman divergence to be $\phi(x)=\sum x\log x$ or choosing the entropy regularizer in FTRL, Hedge/MWU can be obtained. Recent work \citep{PO4MG:22} proposes a unified framework for stochastic games under OFTRL, and another concurrent work MMD \citep{MMD:22} proposes a unified algorithm based on MD to solve RL and two-player zero-sum games. Although the connection between FTRL and MD is very close, MD sometimes results in coupled learning dynamics (an implicit update form), i.e., the player is aware of the utility function or policy of himself or even the opponent, while FTRL often updates in uncoupled learning dynamics. Nevertheless, FTRL and MD can both be both coupled and uncoupled, where the uncoupled dynamics only update policy on the observed or learned utility vector.

\subsection{Entropy regularized RL and regularizer in games} 

The entropy term in single-agent RL plays the role for better exploration or to guarantee a unique solution \citep{SQL:17,SAC,maxEntIRL:08,gail:16}. RAC \citep{RAC_19} studies a general framework for regularized MDP and analyses some properties of regularized MDP. Kullback-Leibler (KL) divergence, which is the relative entropy, is also widely used in RL \citep{TRPO:15,PPO:17}. The philosophy of KL regularization is to avoid the large step size of the policy gradient and guarantee monotone policy improvement. While for games, the notion of quantal response equilibrium (QRE) \citep{QRE:1995} is introduced in entropy regularized normal-form games. Smooth fictitious play \citep{sFP:1993} proposed a regularized version of the fictitious play (FP) \citep{brown:fp:1951} and becomes easier to analyse the learning dynamics. The regularization in smooth FP is a random utility model and is well studied learning in games \citep{learing_in_games:98}. More general regularization may refer to \citep{reg_games_rl:16,ent_games:19}. 

It should be noted that for stochastic games, no-regret algorithms for normal-form games can not be readily applied. The approximate dynamic programming (ADP)~\citep{adp_sg.icml:2015} extends value iteration in RL to stochastic games, thus we can focus on the \textit{per state} regret. 
Besides the theoretical research, regret minimization techniques with deep neural networks have been applied to facilitate the training of RL\citep{RM_in_RL:18,regret_pg:18}. However, to our knowledge, current regret minimization algorithms on continuous control tasks are less competitive than state-of-the-art RL algorithms.

\section{Methods}
\label{sec_methods}
We first focus on the single state $s$ to derive the objective. Akin to QRE, we solve the following regularized objective ($s$ is omitted) in this work:
{
\begin{eqnarray}
\min_{\pi_2}\max_{\pi_1}Q^{\pi_1,\pi_2}+\alpha\mathcal{H}(\pi_1)-\alpha\mathcal{H}(\pi_2)\label{qre_obj},
\end{eqnarray}}
where $Q^{\pi_1,\pi_2}=\int\int\pi_1\pi_2 Q(a_1,a_2){\mathrm{d}a_1}{\mathrm{d}a_2}$.
Consider the game is played for multiple iterations, where both players evolve their policies respectively. Denote
$$
Q^1_t(a_1)=\int \pi_{2,t}(a_2)Q(a_1,a_2)\mathrm{d}a_2,$$
and
$$
Q^2_t(a_2)=\int \pi_{1,t}(a_1)Q(a_1,a_2)\mathrm{d}a_1.
$$

We can perform the MD update to get $\pi_{1,t+1}$ and $\pi_{2,t+1}$ as
{
\begin{align}
&\arg\max_{\pi_1}\langle\pi_1(a),Q^1_t(a)\rangle+\alpha\mathcal{H}(\pi_1)-\frac{KL(\pi_1,\pi_{1,t})}{\eta}\label{max_player_obj},\\
&\arg\min_{\pi_2}\langle\pi_2(a),Q^2_t(a)\rangle-\alpha\mathcal{H}(\pi_2)+\frac{KL(\pi_2,\pi_{2,t})}{\eta} \label{min_player_obj}
\end{align}
}
for the max player and the min player respectively, where $Q^i_t(a)=\int_{a_{-i}}\pi_{-i,t}\cdot Q_t(a_i,a_{-i})\mathrm{d}{a_{-i}}$. Let $F_1=\langle\pi_1(a),Q^1_t(a)\rangle+\alpha\mathcal{H}(\pi_1)-\frac{KL(\pi_1,\pi_{1,t})}{\eta}$ and $F_2=\langle\pi_2(a),Q^2_t(a)\rangle-\alpha\mathcal{H}(\pi_2)+\frac{KL(\pi_2,\pi_{2,t})}{\eta}$. The Eq.\eqref{min_player_obj} can be converted to a maximization problem by adding a minus sign to $F_2$, thus we can solve for the joint objective $F=(F_1,-F_2)$ and the joint policy $\pi_{t}=(\pi_{1,t},\pi_{2,t})$.
The maximum integral functional can be solved by the Euler-Lagrange equation, and we can get
\begin{align}
\pi_{t+1}(a) & \propto \exp(\frac{\eta Q_t(a)+\log\pi_t(a)}{\eta\alpha+1}) \label{md_obj}\\
& =\pi_t(a)\exp(\frac{\eta Q_t(a)}{\eta\alpha+1})\exp(\frac{-\eta\alpha\log\pi_t(a)}{\eta\alpha+1}) \label{factorized_md_obj}\\
&\overset{\eta^\prime=\frac{\eta}{\eta\alpha+1}}{=}\pi_t(a)\exp(\eta^\prime Q_t(a))\exp(-\eta^\prime\alpha\log\pi_t(a)) \label{replaced_fac_md_obj}
\end{align}

The above formulation is closely related to the FTRL, especially MWU, from the no-regret learning perspective. The MWU algorithm updates in a simple incremental way that only depends on the expected performance and policy from the previous round, while other no-regret learning dynamics may depend on all the historical expected performances and policies. Denote $\nu_t=Q_t(a)$ on a given state $s$ for simplicity. Then for the regularized objective Eq.\eqref{qre_obj}, the behavior policy $\pi_{t+1}$ on a given state $s$ in the next iteration satisfies
\begin{eqnarray}
\pi_{t+1}(a) = \arg\max_{\pi} \int_{\mathrm{d}a}\pi(a)(\sum_{k=1}^{t} \nu_{k}-\alpha\log\pi_k-\frac{\log{\pi(a)}}{\eta}) \label{hedge_dynamics}
\end{eqnarray}
Similarly, we can derive the maximum of the above integral functional as follows:
\begin{align}
 \pi_{t+1}(a)&=\frac{\exp(\eta(\sum_{k=1}^{t}\nu_k-\alpha\log\pi_t(a)))}{\int_{\mathrm{d}a}\exp(\eta(\sum_{k=1}^{t}\nu_k-\alpha\log\pi_t(a)))} \label{original_rm_obj} \\
&=\frac{\pi_{t}(a)\exp(\eta(\nu_t-\alpha\log\pi_t(a)))}{\int_{\mathrm{d}a}\pi_{t}(a)\exp(\eta(\nu_t-\alpha\log\pi_t(a)))} \label{opt_obj} \\
&\propto \pi_t(a)\exp(\eta(\nu_t-\alpha\log\pi_t(a))) \\
&=\pi_t(a)\exp(\eta \nu_t)\exp(-\eta\alpha\log\pi_t(a))\label{reduced_obj}.
\end{align}
When $\alpha=0$ in Eq.\eqref{replaced_fac_md_obj} and Eq.\eqref{reduced_obj}, both equations recover standard MWU algorithm. Otherwiese, Eq.\eqref{replaced_fac_md_obj} and Eq.\eqref{reduced_obj} are equivalent with a rescaling of $\eta$.


We can use Kullback-Leibler divergence to approximate $\pi_{t+1}$ (Eq.\eqref{md_obj}) as follows:
\begin{eqnarray}
\pi^{t+1}=\mathop{\text{argmin}}_{\pi}D_{\text{KL}}(\pi||\frac{\exp(\frac{\eta Q_t(a)+\log\pi_t}{\eta\alpha+1})}{Z(s)})  \label{kl_obj},
\end{eqnarray}
where $Z(s)$ is the normalization function.
Parameterizing $\pi,Q$ with $\theta,\phi$ respectively, expanding the goal in \eqref{kl_obj} and ignoring $Z(s)$ which is a constant for $\theta$, the overall objective (multiplied by $\frac{\eta\alpha+1}{\eta}$) becomes
\begin{equation}
\setlength\abovedisplayskip{0.2cm}
\setlength\belowdisplayskip{0.2cm}
\begin{aligned}
J(\pi_\theta) = \mathbb{E}_{a\sim \pi_\theta}[&\frac{1}{\eta}(\log\pi_\theta(a)
- \log\pi_{t}(a)) - Q_t(s,a)+\alpha\log\pi_{\theta}(a)]. \label{param_obj}
\end{aligned}
\end{equation}
The parameterized objective for Eq.\eqref{reduced_obj} (multiplied by $\frac{1}{\eta}$) can be obtained by replacing the $\alpha\log\pi_{\theta}$ in Eq.\eqref{param_obj} with $\alpha\log\pi_{t}$. Eq.\eqref{param_obj} can be viewed as an eager update of the regularized $Q$ function.
When the action space is discrete, we can enumerate every valid action when calculating the inner expectation. If the action space is continuous, it is common to use the reparameterization trick which results in lower variance estimation. Thus, the action $a$ is transformed with the output of the policy network from a noisy vector $\epsilon$, i.e., 
\begin{equation*}
    \setlength\abovedisplayskip{0.2cm}
\setlength\belowdisplayskip{0.2cm}
a=f_{\psi}(\epsilon;s),
\end{equation*}
where $\epsilon$ is usually sampled from a fixed Gaussian distribution. The objective resembles the form of SAC \citep{SAC} in some aspects, except that the objective in Eq.\eqref{param_obj} contains an extra $\log\pi_\theta-\log\pi_t$, which is an estimation of KL. The objective Eq.\eqref{param_obj} more naturally extends SAC with a fixed $\alpha$ and the KL regularizer.

When both players follow MD to act, we have the following theorem
\begin{thm} \label{vanish_kl}
When the action space is compact and if the two players follow Eq.~\eqref{max_player_obj}, Eq.~\eqref{min_player_obj} respectively and $\forall a,\pi,0\le\pi(a)<\infty$, with the choice of the $\eta\le\min\{\frac{1}{b^2},\frac{\alpha^2}{b^2L^4}\}$, the following relations holds:
\begin{equation}
KL(\pi^\star,\pi_{t})\le (\frac{1}{1+\eta\alpha})^t KL(\pi^\star,\pi_0),
\end{equation}
where $\pi^\star$ is the solution to Eq.\eqref{qre_obj}, $\pi_t=(\pi_{1,t},\pi_{2,t})$ and $L=\frac{R_{\text{max}}}{1-\gamma}$.
\end{thm}
Theorem \ref{vanish_kl} discloses that $\alpha>0$ is necessary for the last-iterate convergence. A non-zero $\alpha$ also guarantees that the $\pi^\star$ is a Boltzmann policy that is bounded, thus $KL(\pi^\star,\pi_0)$ can be bounded if $\pi_0$ is the uniform distribution. This implies the convergence to a QRE in two-player zero-sum games.
On the other hand, $\alpha$ should be small to make the regularized optimal solution close to the original. $\eta$ should be very small to strictly match the condition, however, a too small $\eta$ will incur a very large $1/\eta$, which can be numerically unstable, thus we simply let $1/\eta$ much larger than $\alpha$. Here, both $\alpha$ and $\eta$ are fixed. A time-vanishing $\alpha$ can also be adopted as done in \citep{rew_trans_21,EG_pu_omwu:21}.

Note that when $\pi_t$ is a density function on a compact set, it turns into solving an extremum of a functional rather than solving for the probability vector in discrete action space. Thus, the objective function $F_1,F_2$ are not necessarily concave/convex with respect to the action $a$. We need to slightly fix some of the proof steps that are direct conclusions from the convex and strongly convex properties. See Appendix \ref{detailed_proof} for the detailed proof.
\begin{algorithm}[H]
    \caption{Policy Optimization with Regularized Leader}
    \label{orm}
    \begin{algorithmic}[1] 
        \State Initialize $t=0, \pi_\theta=\pi^{t}_\theta, Q^t_\phi, Q^{\text{target}}_{\bar{\phi}}$, target network mixture coefficient $\tau$, learning rate $\alpha$ and $\beta$ 
        \For{each iteration $t$}
        \For{$T$ rounds}
            \State Sample $M$ steps with $\pi_\theta$ and store $(s,a,r,s^\prime,a^\prime)$ and whether $s'$ is terminal in Buffer $B$
            \For{ $i=1,\cdots,K_1$}  \Comment{Policy Evaluation}
                \State $\phi \gets \phi - \alpha\nabla_\phi\mathbb{E}_{(s,a,r,s^\prime,a^\prime)\sim B}[(Q^t_\phi(s,a)-(r(s,a)+\gamma\cdot Q^{\text{target}}_{\bar\phi}(s^\prime,a^\prime)-\alpha\log\pi_\theta(s^\prime,a^\prime)))^2]$
                \State $\bar{\phi}^{t}\gets (1-\tau)\bar{\phi}^{t} + \tau \phi^t$
            \EndFor
            \For{ $i=1,\cdots,K_2$}  \Comment{Policy Learning}
                \State $\theta \gets \theta - \beta\nabla_\theta \mathbb{E}_{a\sim \pi_\theta}[\frac{1}{\eta}(\log\pi_\theta(a)
- \log\pi_{t}(a)) - Q_t(s,a)+\alpha\log\pi_{\theta}(a)] $ \Comment{as Eq.\eqref{param_obj}}
            \EndFor
            \EndFor
        \State Update $Q^{t-1} \gets Q^{t}_{\phi}, \pi^{t-1} \gets \pi_\theta$ \label{value_pi_update}
        \State $t \gets t+1$
        \EndFor
    \end{algorithmic}
\end{algorithm}

With approximate dynamic programming techniques \citep{adp_sg.icml:2015}, we can move to RL and stochastic games. The learning process (named PORL) of the joint optimization is summarized in Algorithm \ref{orm}. PORL can be applied in single-agent RL and two-player games. For single-agent RL problem, the min player is simply omitted. In broad strokes, the proposed procedure consists of two steps: (1) evaluating the expected value function of the behavior policy and then (2) learning the behavior policy for the next iteration.

\begin{figure*}[!h]
    \centering
    \includegraphics[width=0.97\textwidth]{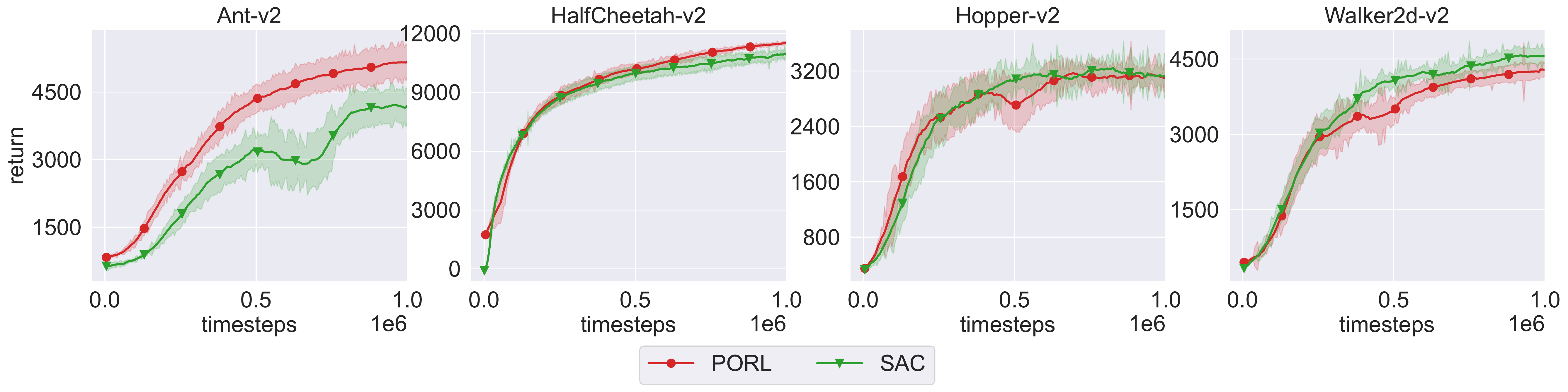}
    \caption{Return curves in $4$ stationary MuJoCo tasks. The curves are shaded with standard errors.}
    \label{fig:stationary_mujoco}
\end{figure*}

\begin{figure*}[!h]
    \centering
    \includegraphics[width=0.97\textwidth]{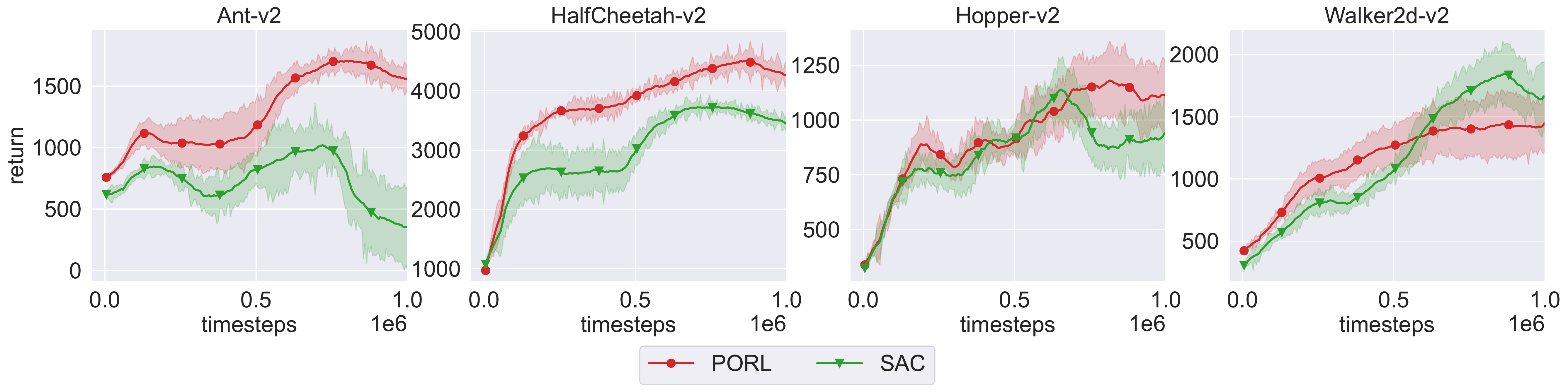}
    \caption{Return curves in $4$ non-stationary MuJoCo tasks The curves are shaded with standard errors. For every $0.125$M training time steps, we change the \texttt{gravity} of the environments.}
    \label{fig:non_stationary_mujoco}
\end{figure*}

\begin{figure*}[!h]
    \centering
    \includegraphics[width=0.97\textwidth]{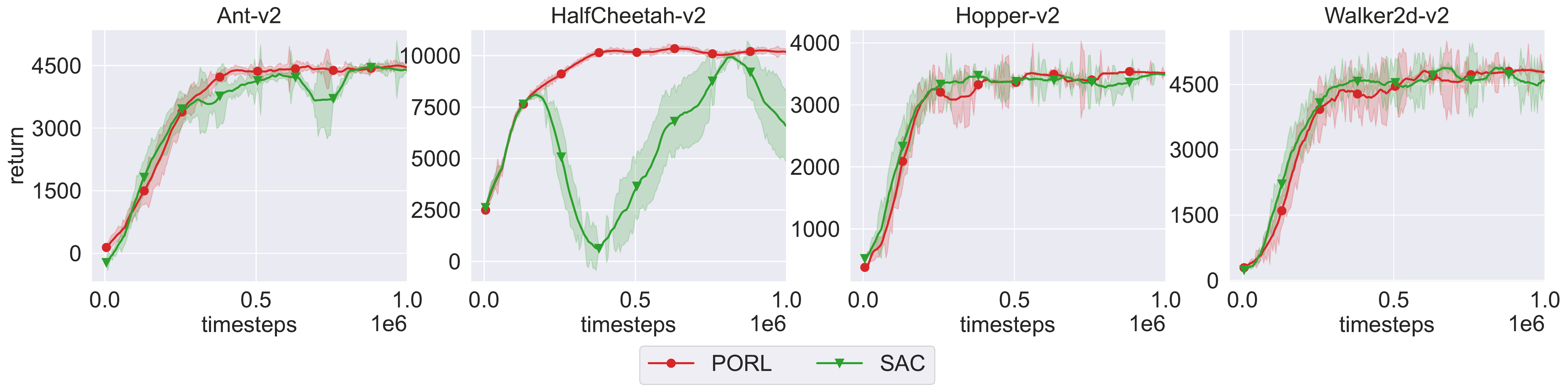}
    \caption{Return curves in $4$ adversarial imitation learning MuJoCo tasks. We replace the policy optimization method in GAIL with PORL and SAC to conduct these comparisons. The curves are shaded with standard errors.}
    \label{fig:gail_mujoco}
\end{figure*}

\section{Experiments}
In this part, we compare PORL with baseline methods in stationary and non-stationary scenarios. We first conducted the experiments in stationary MuJoCo tasks to evaluate the data efficiency of PORL in typical RL benchmarks. Then, we assess the robustness of PORL in three non-stationary scenarios, i.e., dynamical environments, adversarial training environments, and competitive games. Each experiment is conducted with $5$ distinct random seeds unless explicitly mentioned in the text. 

\textbf{Continuous MuJoCo tasks.} First, we evaluate the proposed algorithm on four stationary Gym-MuJoCo tasks. Since these environments are stationary, the KL constraint is relaxed a little ($1/\eta=0.1,\alpha=0.2, T=1000$) in PORL. For SAC, $\alpha$ is auto-tuned as \citep{SAC_auto_alpha}.

We instantiate the policy network and the value networks by MLPs with 2 hidden layers. The hidden units are $[256, 64]$. The activation functions of the hidden layers are \texttt{leaky\_relu}. No activation function is added to the output layer. The policy outputs a mean vector $\mathbf{\mu}$ and a logarithmic standard deviation $\mathbf{\sigma}$ of a diagonal Gaussian distribution $\mathcal{N}(\cdot\mid \mathbf{\mu}, \mathbf{\sigma})$, from which an action vector will be sampled. After squashing the action vector with \texttt{tanh} to make the action falls into the valid compact action space, we will apply the squashed action to the environment. This resembles SAC~\citep{SAC} as the authors showed that squashing can stabilize the training process. We use Adam~\citep{adam.iclr:2015} as the optimizer. The learning rates of the policy network and the value networks are $3e-4$ and $1e-3$, respectively.
Every policy is trained with $1$M environment steps. 

Figure \ref{fig:stationary_mujoco} shows the performance of PORL in $4$ stationary MuJoCo tasks. We can find that PORL is much better on Ant-v2. The learning process shows few or no sudden performance drops. The performance is competitive to SAC, a SOTA reward maximization algorithm on the other MuJoCo domains. These results imply the potential of PORL for tackling RL problems.

\textbf{Non-stationary environments.} In addition to stationary tasks, we also compare PORL with SAC in non-stationary environments. The hyper-parameters for both algorithms keep exactly the same. We construct $4$ dynamical MuJoCo tasks by changing the parameters (the gravity) of the environment dynamics periodically. Specifically, before training, we opt $8$ values, $\{-1, -2, -5, -10, -15, -20, -25, -30\}$, as the candidate environment parameters, thus each non-stationary task consists of $8$ stationary environments. For every $0.125$M time steps, we change the \texttt{gravity} of the environment to one of the candidate values. Every value is sampled without replacement during the full training process. In Figure~\ref{fig:non_stationary_mujoco}, we report the comparison of PORL and SAC in terms of the averaged return in the $8$ environments. Both PORL and SAC suffer from performance drops because of the non-stationarity. However, we find that the performance drop of SAC is significantly heavier except on Walker2d. 
These results reveal the robustness of PORL to the environment non-stationarity, which is benefited from the no-regret and last-iterate convergence properties.

\textbf{Adversarial environments.} We also test PORL in the scenario of adversarial training. Specifically, we compare PORL with SAC in the framework of adversarial imitation learning (AIL), which is a fully adversarial training scenario. The target of AIL is to recover the expert policy from a set of expert data. AIL learns two networks simultaneously, a policy network and a discriminator network. The discriminator is trained to discriminate the expert data from the \textit{fake} data generated by the policy. While the policy (the generator) learns to mimic the expert to fool the discriminator. In each task, we first train a policy by SAC and use the policy at convergence to collect the expert dataset. The performance of the expert data is listed in Table~\ref{tab:expert_performance}. We use GAIL~\citep{gail:16} as the AIL algorithm. GAIL typically uses TRPO~\cite{TRPO:15} or PPO~\cite{PPO:17} as the policy optimization methods. In order to compare the performance of PORL and SAC in the AIL scenario, we replace the policy optimization method in GAIL with PORL and SAC. The hyper-parameter for SAC keeps the same as stationary and non-stationary MuJoCo tasks. Since the dynamics changes faster, we take $\alpha=0.01$, $1/\eta=0.1$ and the gradient updates $T=10$ for adversarial training.

\begin{table}[H]
    \centering
        \caption{Environment return of the expert data.}
    \begin{tabular}{cccc}\toprule
Ant-v2&HalfCheetah-v2&Hopper-v2&Walker2d-v2\\\midrule
4813.04&10545.93&3557.87&3975.29\\
\bottomrule
    \end{tabular}
    \label{tab:expert_performance}
\end{table}

\begin{table}[H]
    \centering
\caption{Final return $\pm$ standard errors in 4 stationary MuJoCo tasks in the adversarial imitation learning scenarios. We replace the policy optimization method in original GAIL with PORL and SAC to make these comparisons.}
    \resizebox{0.47\textwidth}{!}{
    \begin{tabular}{l|r@{~$\pm$~}lr@{~$\pm$~}l}\toprule
& \multicolumn{2}{c}{PORL} & \multicolumn{2}{c}{SAC}\\\midrule
Ant-v2 & $ \mathbf{4450.39} $ & $ \mathbf{103.9} $ & $4384.35$& $96.62$\\
HalfCheetah-v2 & $ \mathbf{10179.53} $ & $ \mathbf{177.93} $ & $6575.59$& $1724.85$\\
Hopper-v2 & $ \mathbf{3509.43} $ & $ \mathbf{59.8} $ & $3494.53$& $54.44$\\
Walker2d-v2 & $ \mathbf{4790.49} $ & $ \mathbf{36.75} $ & $4573.75$& $19.14$\\
\bottomrule
    \end{tabular}
    }
    \label{tab:gail_compare}
\end{table}

We present the return curves and the final returns of the policies in Figure~\ref{fig:gail_mujoco} and Table~\ref{tab:gail_compare}. The result reveals that PORL is superior to SAC in adversarial scenarios. From Figure~\ref{fig:gail_mujoco}, it can be observed that as the training proceeds, PORL achieves a more stable performance on Ant and HalfCheetah. This result also discloses that PORL takes the advantage of the no-regret and the last-iterate convergence property, and thus is more robust in adversarial training scenarios. 

\textbf{Remarks on MuJoCo domains.} We follow the same hyper-parameters of SAC for PORL if they have exactly the same functionality. The extra hyper-parameters $\alpha,\eta$, and $T$ are mostly decided based on the theory part and the changing speed of the dynamics. The robustness of SAC in some of the non-stationary MuJoCo environments and adversarial training scenarios is intriguing. Besides, the Eq.\eqref{param_obj} can also be implemented in an on-policy way without changing too much, which resembles the objective of PPO as done in MMD \citep{MMD:22}. However, PORL extends SAC which achieves higher rewards and has shown the potential and benefit of no-regret learning in both stationary and non-stationary environments for continuous action domains.

\textbf{Continuous action space competitive environments.} Finally, we choose two competitive scenarios from the multi-agent particle environments (MPE) \citep{mpe17}, i.e., \texttt{SimplePush} and \texttt{SimpleAdversary}. In \texttt{SimplePush}, the main agent (player 2) is rewarded based on the Euclidean distance to a landmark. The adversary (player 1) is rewarded if it is close to the landmark, and if the agent is far from the landmark. So the adversary learns to push the agent away from the landmark. In \texttt{SimpleAdversary}, there are two landmarks, and the main agent (player 2) is rewarded based on how close one of them is to the target landmark, but negatively rewarded if the adversary (player 1) is close to the target landmark. The adversary is rewarded based on how close it is to the target, but it doesn’t know which is the target landmark. So the main agent has to learn to split up and cover all landmarks to deceive the adversary, which is the most strategic policy among all these agents. The original environment contains a main window of size $[-1,1]\times[-1,1]$ and the agents are allowed to exit the window. The landmarks and the agents are randomly spawned on the main window. We slightly modify the original scenarios to make the two environments more competitive. The number of agents is set as two, and they are restricted to stay in a region $[-2,2]\times[-2,2]$. The environment terminates when one of the agents exits the region or the time steps reach 100. These two scenarios are not strictly two-player zero-sum games, but are fully competitive.

We compare PORL with SAC and NFSP in the competitive games. Instead of fully independent learning, we alternate the learning player every 1000 gradient updates (i.e., $1K$ time steps) and the total steps is $1M$. The policies and value networks are both MLPs with 2 hidden layers, which do not share parameters. The hidden units are $[64,64]$ on two MPE scenarios. Since the main agent and the adversary are not symmetric, we use separate parameters for the two players. The learning rates are $1e-3$ for both networks. We set the policy update interval $T=10$, which better fits MPE tasks. The value of $\alpha$ in SAC is auto-tuned, while we set the initial $\alpha_0=0.01$ and slowly decay $\alpha=\min(0.999*\alpha,0.1\alpha_0)$ every 1000 gradient updates, and $\eta=0.1$ for PORL. We also evaluated $\alpha=0$ for PORL to verify the necessity of a non-zero $\alpha$. We do find the PORL performs a bit worse when $\alpha=0$ (see Appendix \ref{abaltion_study} for details). For NFSP, we replace DQN by SAC as the best response for continuous action space, and anticipatory parameter $\eta=0.2$ which is the best chosen from $\{0.1,0.2,0.5\}$ (see Appendix \ref{abaltion_study} for the full learning curves). The buffer for the average policy is 10 times the size of the replay buffer as done in \citep{nfsp_heinrich:16}. For NFSP, we evaluate the average policy. The learning curves of the best hyper-parameters are shown in Fig.\ref{fig:mpe}. Both PORL and SAC converge stably, while the average policy of NFSP is rather unstable. We then evaluate the converged policies of both algorithms. As the two scenarios are asymmetric, we use a cross-play evaluation protocol. We let player 1 trained by algorithm 1 (\texttt{Alg1}) pit against player 2 trained by algorithm 2 (\texttt{Alg2}) for 200 episodes and record the returns of each player, and repeat the above process by exchanging the players trained from the algorithms. Then the total returns obtained by the same algorithm is the final score. We report the final results in Table.\ref{tab:mpe_pit_score}. PORL can obtain higher scores when pitting against either SAC or NFSP in both scenarios. We notice that the main agent of \texttt{SimpleAdversary} trained by PORL has a larger advantage over the main agents of SAC and NFSP. This reveals that PORL is more likely to deceive the adversaries than SAC and NFSP in competitive games. Detailed scores can be found in Appendix \ref{param_detail}.

\begin{figure}[H]
    \centering
    \includegraphics[width=0.7\textwidth]{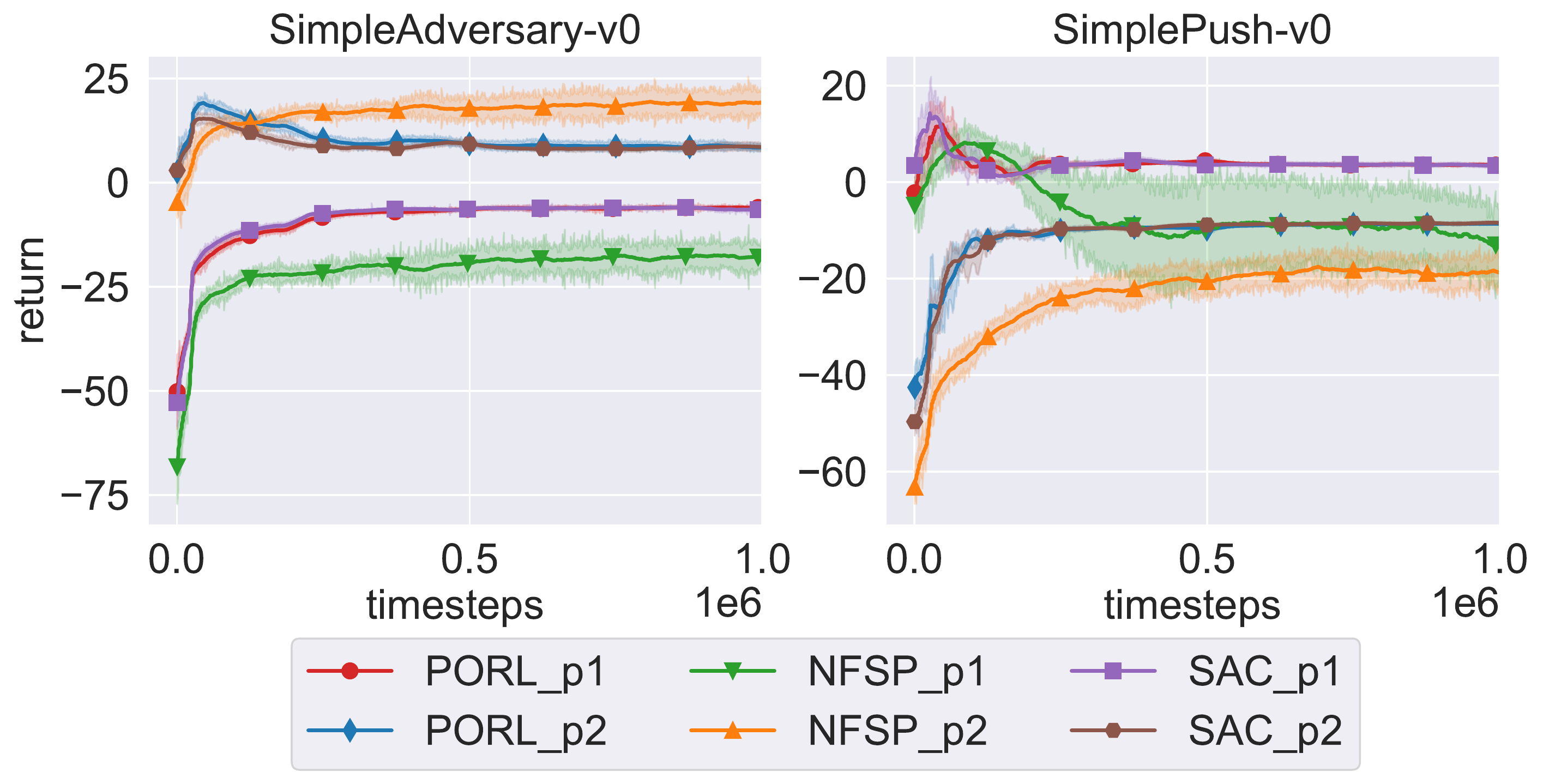}
    \caption{Return curves in $2$ MPE tasks. The returns for player 1 (the adversary) and 2 (the main agent) are plotted respectively. The curves are shaded with standard errors.}
    \label{fig:mpe}
\end{figure}

\begin{table}[H]
    \centering
    \caption{Cross-play returns on MPE environments. Push and Adv stand for \texttt{SimplePush} and \texttt{SimpleAdversary} tasks respectively. In each task, the score $S_{\text{row}, \text{col}}$ denotes the total returns of $\text{Alg}_\text{row}$ pits against $\text{Alg}_\text{col}$. If $S_{\text{row}, \text{col}}>S_{\text{col}, \text{row}}$, then $\text{Alg}_\text{row}$ is better than $\text{Alg}_\text{col}$.} 
\scalebox{0.81}{
\begin{tabular}{l|cccc}
\toprule
Task &  & PORL & NFSP & SAC\\
\midrule
\multirow{3}{1cm}{Push}
& PORL & - & $12.36\pm 4.85$ & $-3.30\pm 0.97$\\
& NFSP & $-49.35\pm 9.75$ & - & $-50.11\pm 10.92$\\
& SAC & $-6.06\pm 1.67$ & $11.64\pm 3.89$ & -\\
\midrule
\multirow{3}{1cm}{Adv}
& PORL & - & $12.86\pm 8.57$  & $5.15\pm 1.29$ \\
& NFSP & $-6.33\pm 6.73$  & - &  $-3.42\pm 8.91$ \\
& SAC  & $0.72 \pm 2.72$  & $11.74\pm 9.38$  & -\\


\bottomrule
\end{tabular}}
    \label{tab:mpe_pit_score}
\end{table}

\begin{figure}[!h]
    \centering
    \includegraphics[width=0.7\textwidth]{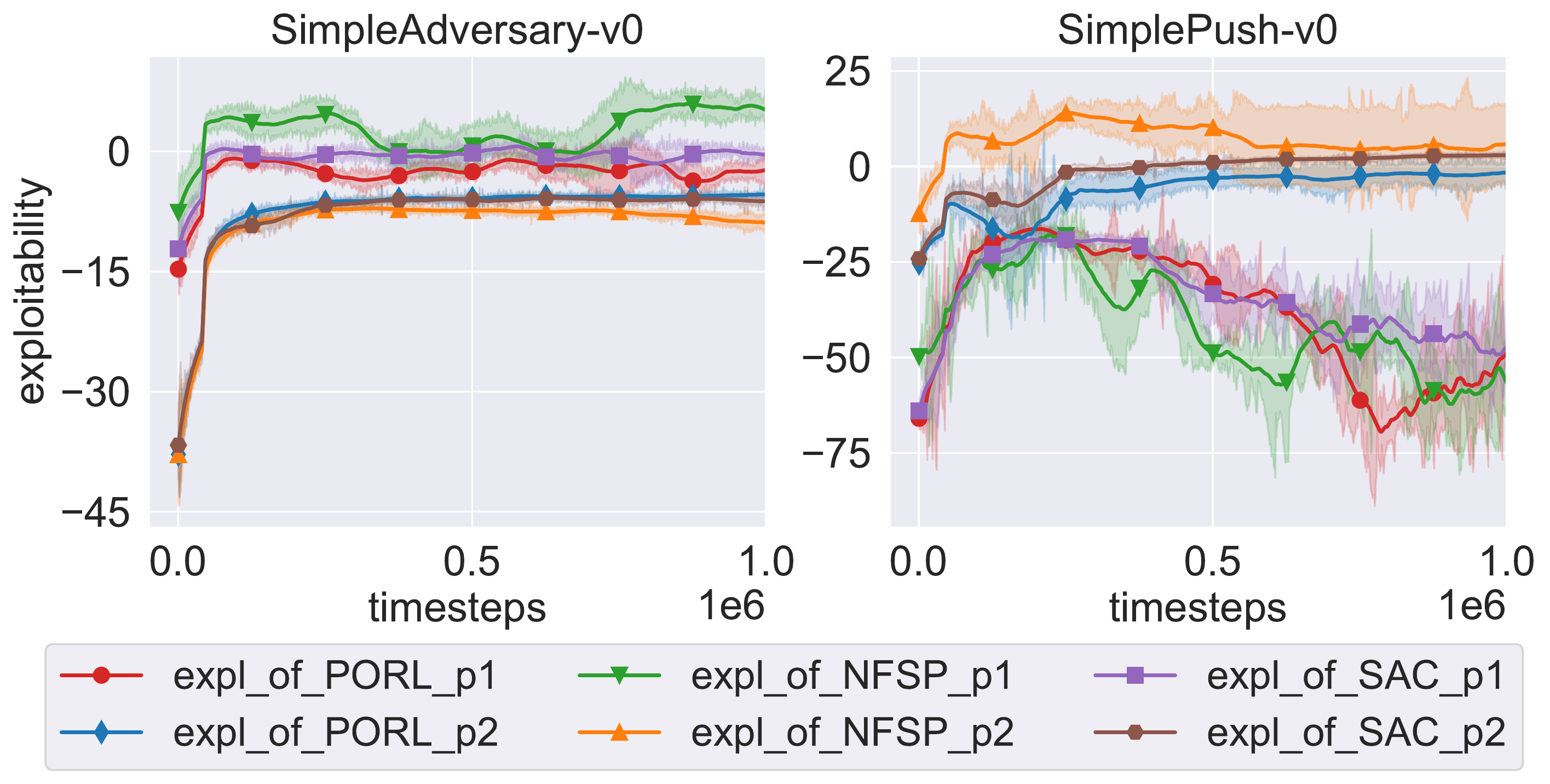}
    \caption{Approximate exploitability learning curves in $2$ multi-agent particle environment competitive tasks. The exploitability of player 1 and 2 are plotted respectively. A lower average exploitability means the corresponding algorithm is more robust against exploiters.}
    \label{fig:expl_mpe}
\end{figure}
We also invoke a SAC as the best response oracle to get the approximate exploitability. However, we observed SAC fails to exploit the trained agents well, especially the main agent on \texttt{SimplePush}. Nevertheless, the learning curves of the exploiter in Fig.\ref{fig:expl_mpe} demonstrate that PORL is harder to be exploited. We argue the slow convergence of NFSP is due to the mixed learning of the best response oracle and the average policy. In NFSP, the main behavior policy is the average policy (with probability $1-\eta$, where $\eta$ should be small). Following the practical observations in \citep{BCQ}, if the behavior policy and the learning policy are very different, off-policy RL algorithms cannot efficiently learn the optimal policy, thus samples for the average policy are not from the real BR.

\noindent{\textbf{Ablation studies.}} We ablate two hyper-parameters in PORL, i.e., the gradient steps for each iteration ($T$) and $\eta$. We evaluate different hyper-parameters in non-stationary HalfCheetah-v2. We first fix $\eta$ and vary $T$ to different values. We report the final return of the policy trained with $1$M environment interactions in Table~\ref{tab:ablation_t} in Appendix \ref{abaltion_study}. Each experiment is repeated with $3$ seeds. We can find the performance will benefit from a smaller $T$. The smallest $T$ ($10$) obtains the highest final return. This result originates from the fact that $\log \pi_\theta(a)-\log \pi_t(a)$ in Eq.\eqref{param_obj} will be large when $\pi_t(a)$ is very old. A larger discrepancy could introduce instability in sample-based policy learning. We then fix $T$ and vary $\eta$. The result is summarized in Table~\ref{tab:ablation_eta} in Appendix \ref{abaltion_study}. The experiments are also repeated with $3$ seeds. We find that PORL is relatively robust to $\eta$, and $1/\eta=0.1$ performs well in all the experiments.

\section{Conclusion}
In this paper, we propose a practical no-regret style RL algorithm, PORL, to solve stationary MDP, dynamic MDP, and adversarial training, in continuous action space. We show PORL can extend to handle continuous action space and has the last-iterate convergence property. The experimental results demonstrate the proposed PORL is superior (at least competitive) to SAC and more robust in stationary, dynamical environments and adversarial training scenarios. PORL is superior to SAC and NFSP in competitive games.

We currently used a small and fixed $\alpha$ for the regularized MDP which only converges to the QRE. In the future, we will investigate the vanishing version. Besides, the regularization term is limited to the entropy, which has some nice properties in implementation. Other regularizers may also be available.

\bibliographystyle{abbrvnat}
\bibliography{deeprm}

\begin{thebibliography}{47}
\providecommand{\natexlab}[1]{#1}
\providecommand{\url}[1]{\texttt{#1}}
\expandafter\ifx\csname urlstyle\endcsname\relax
  \providecommand{\doi}[1]{doi: #1}\else
  \providecommand{\doi}{doi: \begingroup \urlstyle{rm}\Url}\fi

\bibitem[Aung et~al.(2021)Aung, Wang, Yu, An, Jayavelu, and Li]{DO_GAN:22}
A.~P.~P. Aung, X.~Wang, R.~Yu, B.~An, S.~Jayavelu, and X.~Li.
\newblock {DO-GAN:} {A} double oracle framework for generative adversarial
  networks.
\newblock \emph{CoRR}, 2021.

\bibitem[Brown(1951)]{brown:fp:1951}
G.~W. Brown.
\newblock Iterative solutions of games by fictitious play.
\newblock In \emph{Activity Analysis of Production and Allocation}, Wiley, NY,
  1951.

\bibitem[Cai et~al.(2022)Cai, Oikonomou, and Zheng]{last_iter_eg_ogda:22}
Y.~Cai, A.~Oikonomou, and W.~Zheng.
\newblock Tight last-iterate convergence of the extragradient method for
  constrained monotone variational inequalities.
\newblock \emph{CoRR}, abs/2204.09228, 2022.

\bibitem[Cen et~al.(2021)Cen, Wei, and Chi]{EG_pu_omwu:21}
S.~Cen, Y.~Wei, and Y.~Chi.
\newblock Fast policy extragradient methods for competitive games with entropy
  regularization.
\newblock In \emph{Advances in Neural Information Processing Systems 34}, pages
  27952--27964, virtual, 2021.

\bibitem[Cesa{-}Bianchi and Lugosi(2006)]{online_and_games:06}
N.~Cesa{-}Bianchi and G.~Lugosi.
\newblock \emph{Prediction, learning, and games}.
\newblock Cambridge University Press, 2006.

\bibitem[Chen and Peng(2020)]{faster_hedge:20}
X.~Chen and B.~Peng.
\newblock Hedging in games: Faster convergence of external and swap regrets.
\newblock In \emph{Advances in Neural Information Processing Systems 33},
  virtual, 2020.

\bibitem[Daskalakis and Panageas(2019)]{last_iterative_OMWU:19}
C.~Daskalakis and I.~Panageas.
\newblock Last-iterate convergence: Zero-sum games and constrained min-max
  optimization.
\newblock In \emph{10th Innovations in Theoretical Computer Science
  Conference}, San Diego, California, 2019.

\bibitem[Farina et~al.(2021)Farina, Kroer, and Sandholm]{pred_blackwell:21}
G.~Farina, C.~Kroer, and T.~Sandholm.
\newblock Faster game solving via predictive blackwell approachability:
  Connecting regret matching and mirror descent.
\newblock In \emph{Thirty-Fifth {AAAI} Conference on Artificial Intelligence},
  Virtual Event, 2021. {AAAI} Press.

\bibitem[Fudenberg and Kreps(1993)]{sFP:1993}
D.~Fudenberg and D.~M. Kreps.
\newblock Learning mixed equilibria.
\newblock \emph{Games and Economic Behavior}, 5\penalty0 (3):\penalty0
  320--367, 1993.

\bibitem[Fudenberg and Levine(1998)]{learing_in_games:98}
D.~Fudenberg and D.~K. Levine.
\newblock \emph{The Theory of Learning in Games}.
\newblock MIT Press, 1998.

\bibitem[Fujimoto et~al.(2019)Fujimoto, Meger, and Precup]{BCQ}
S.~Fujimoto, D.~Meger, and D.~Precup.
\newblock Off-policy deep reinforcement learning without exploration.
\newblock In \emph{Proceedings of the 36th International Conference on Machine
  Learning}, volume~97 of \emph{Proceedings of Machine Learning Research},
  pages 2052--2062, Long Beach, California, 2019.

\bibitem[Ge et~al.(2018)Ge, Xia, Chen, Berry, and Wu]{fictitiousGAN:18}
H.~Ge, Y.~Xia, X.~Chen, R.~Berry, and Y.~Wu.
\newblock Fictitious {GAN:} training gans with historical models.
\newblock In \emph{Computer Vision 15th European Conference}, Munich, Germany,
  2018.

\bibitem[Goodfellow et~al.(2014)Goodfellow, Pouget{-}Abadie, Mirza, Xu,
  Warde{-}Farley, Ozair, Courville, and Bengio]{gan:14}
I.~J. Goodfellow, J.~Pouget{-}Abadie, M.~Mirza, B.~Xu, D.~Warde{-}Farley,
  S.~Ozair, A.~C. Courville, and Y.~Bengio.
\newblock Generative adversarial nets.
\newblock In \emph{Advances in Neural Information Processing Systems 27},
  Montreal, Quebec, Canada, 2014.

\bibitem[Haarnoja et~al.(2017)Haarnoja, Tang, Abbeel, and Levine]{SQL:17}
T.~Haarnoja, H.~Tang, P.~Abbeel, and S.~Levine.
\newblock Reinforcement learning with deep energy-based policies.
\newblock In \emph{Proceedings of the 34th International Conference on Machine
  Learning}, Sydney, NSW, Australia, 2017. {PMLR}.

\bibitem[Haarnoja et~al.(2018{\natexlab{a}})Haarnoja, Zhou, Abbeel, and
  Levine]{SAC}
T.~Haarnoja, A.~Zhou, P.~Abbeel, and S.~Levine.
\newblock Soft actor-critic: Off-policy maximum entropy deep reinforcement
  learning with a stochastic actor.
\newblock In \emph{Proceedings of the 35th International Conference on Machine
  Learning}, Stockholmsm{\"{a}}ssan, Stockholm, Sweden, 2018{\natexlab{a}}.
  {PMLR}.

\bibitem[Haarnoja et~al.(2018{\natexlab{b}})Haarnoja, Zhou, Hartikainen,
  Tucker, Ha, Tan, Kumar, Zhu, Gupta, Abbeel, and Levine]{SAC_auto_alpha}
T.~Haarnoja, A.~Zhou, K.~Hartikainen, G.~Tucker, S.~Ha, J.~Tan, V.~Kumar,
  H.~Zhu, A.~Gupta, P.~Abbeel, and S.~Levine.
\newblock Soft actor-critic algorithms and applications, 2018{\natexlab{b}}.
\newblock URL \url{http://arxiv.org/abs/1812.05905}.

\bibitem[Hannan(1957)]{hannan_consistency}
J.~Hannan.
\newblock Approximation to bayes risk in repeated play.
\newblock \emph{Contributions to the Theory of Games}, 3:\penalty0 97--139,
  1957.

\bibitem[Hart and Mas-Colell(2000)]{rm.Hart_Mas:2000}
S.~Hart and A.~Mas-Colell.
\newblock A simple adaptive procedure leading to correlated equilibrium.
\newblock \emph{Econometrica}, 68\penalty0 (7540):\penalty0 1127–1150, 2000.

\bibitem[Hazan(2016)]{intro_to_OCO:16}
E.~Hazan.
\newblock Introduction to online convex optimization.
\newblock \emph{Foundations and Trends in Optimization}, 2\penalty0
  (3-4):\penalty0 157--325, 2016.

\bibitem[Heinrich and Silver(2016)]{nfsp_heinrich:16}
J.~Heinrich and D.~Silver.
\newblock Deep reinforcement learning from self-play in imperfect-information
  games, 2016.

\bibitem[Hennes et~al.(2020)Hennes, Morrill, Omidshafiei, Munos, P{\'{e}}rolat,
  Lanctot, Gruslys, Lespiau, Parmas, Du{\'{e}}{\~{n}}ez{-}Guzm{\'{a}}n, and
  Tuyls]{NeuralRD}
D.~Hennes, D.~Morrill, S.~Omidshafiei, R.~Munos, J.~P{\'{e}}rolat, M.~Lanctot,
  A.~Gruslys, J.~Lespiau, P.~Parmas, E.~A. Du{\'{e}}{\~{n}}ez{-}Guzm{\'{a}}n,
  and K.~Tuyls.
\newblock Neural replicator dynamics: Multiagent learning via hedging policy
  gradients.
\newblock In \emph{Proceedings of the 19th International Conference on
  Autonomous Agents and Multiagent Systems}, Auckland, New Zealand, 2020.

\bibitem[Ho and Ermon(2016)]{gail:16}
J.~Ho and S.~Ermon.
\newblock Generative adversarial imitation learning.
\newblock In \emph{Advances in Neural Information Processing Systems 29},
  Barcelona, Spain, 2016.

\bibitem[Jin et~al.(2018)Jin, Keutzer, and Levine]{RM_in_RL:18}
P.~H. Jin, K.~Keutzer, and S.~Levine.
\newblock Regret minimization for partially observable deep reinforcement
  learning.
\newblock In \emph{Proceedings of the 35th International Conference on Machine
  Learning}, Stockholmsm{\"{a}}ssan, Stockholm, Sweden, 2018.

\bibitem[Kalai and Vempala(2005)]{FTL:05}
A.~T. Kalai and S.~S. Vempala.
\newblock Efficient algorithms for online decision problems.
\newblock \emph{Journal of Computer and System Sciences}, 71\penalty0
  (3):\penalty0 291--307, 2005.

\bibitem[Kingma and Ba(2015)]{adam.iclr:2015}
D.~P. Kingma and J.~Ba.
\newblock Adam: {A} method for stochastic optimization.
\newblock In \emph{Proceedings of the 3rd International Conference on Learning
  Representations}, 2015.

\bibitem[Lanctot et~al.(2009)Lanctot, Waugh, Zinkevich, and Bowling]{MCCFR:09}
M.~Lanctot, K.~Waugh, M.~Zinkevich, and M.~H. Bowling.
\newblock Monte carlo sampling for regret minimization in extensive games.
\newblock In \emph{Advances in Neural Information Processing Systems 22}, pages
  1078--1086, Vancouver, British Columbia, Canada, 2009.

\bibitem[Lee et~al.(2021)Lee, Kroer, and Luo]{last_iter_efg:21}
C.~Lee, C.~Kroer, and H.~Luo.
\newblock Last-iterate convergence in extensive-form games.
\newblock In \emph{Advances in Neural Information Processing Systems 34}, pages
  14293--14305, virtual, 2021.

\bibitem[McKelvey and Palfrey(1995)]{QRE:1995}
R.~D. McKelvey and T.~R. Palfrey.
\newblock Quantal response equilibria for normal form games.
\newblock \emph{Games and Economic Behavior}, 10\penalty0 (1):\penalty0 6--38,
  1995.

\bibitem[Mertikopoulos and Sandholm(2016)]{reg_games_rl:16}
P.~Mertikopoulos and W.~H. Sandholm.
\newblock Learning in games via reinforcement and regularization.
\newblock \emph{Mathematics of Operations Research}, 41\penalty0 (4):\penalty0
  1297--1324, 2016.

\bibitem[Mnih et~al.(2015)Mnih, Kavukcuoglu, Silver, Rusu, Veness, Bellemare,
  Graves, Riedmiller, Fidjeland, Ostrovski, Petersen, Beattie, Sadik,
  Antonoglou, King, Kumaran, Wierstra, Legg, and Hassabis]{atari_dqn:15}
V.~Mnih, K.~Kavukcuoglu, D.~Silver, A.~A. Rusu, J.~Veness, M.~G. Bellemare,
  A.~Graves, M.~A. Riedmiller, A.~Fidjeland, G.~Ostrovski, S.~Petersen,
  C.~Beattie, A.~Sadik, I.~Antonoglou, H.~King, D.~Kumaran, D.~Wierstra,
  S.~Legg, and D.~Hassabis.
\newblock Human-level control through deep reinforcement learning.
\newblock \emph{Nature}, 518\penalty0 (7540):\penalty0 529--533, 2015.

\bibitem[Mordatch and Abbeel(2017)]{mpe17}
I.~Mordatch and P.~Abbeel.
\newblock Emergence of grounded compositional language in multi-agent
  populations, 2017.

\bibitem[P{\'{e}}rolat et~al.(2015)P{\'{e}}rolat, Scherrer, Piot, and
  Pietquin]{adp_sg.icml:2015}
J.~P{\'{e}}rolat, B.~Scherrer, B.~Piot, and O.~Pietquin.
\newblock Approximate dynamic programming for two-player zero-sum markov games.
\newblock In \emph{Proceedings of the 32nd International Conference on Machine
  Learning}, Lille, France, 2015. JMLR.

\bibitem[P{\'{e}}rolat et~al.(2021)P{\'{e}}rolat, Munos, Lespiau, Omidshafiei,
  Rowland, Ortega, Burch, Anthony, Balduzzi, Vylder, Piliouras, Lanctot, and
  Tuyls]{rew_trans_21}
J.~P{\'{e}}rolat, R.~Munos, J.~Lespiau, S.~Omidshafiei, M.~Rowland, P.~A.
  Ortega, N.~Burch, T.~W. Anthony, D.~Balduzzi, B.~D. Vylder, G.~Piliouras,
  M.~Lanctot, and K.~Tuyls.
\newblock From poincar{\'{e}} recurrence to convergence in imperfect
  information games: Finding equilibrium via regularization.
\newblock In \emph{Proceedings of the 38th International Conference on Machine
  Learning}, Virtual Event, 2021.

\bibitem[Savas et~al.(2019)Savas, Ahmadi, Tanaka, and Topcu]{ent_games:19}
Y.~Savas, M.~Ahmadi, T.~Tanaka, and U.~Topcu.
\newblock Entropy-regularized stochastic games.
\newblock In \emph{58th {IEEE} Conference on Decision and Control}, pages
  5955--5962, Nice, France, 2019. {IEEE}.

\bibitem[Schulman et~al.(2015)Schulman, Levine, Abbeel, Jordan, and
  Moritz]{TRPO:15}
J.~Schulman, S.~Levine, P.~Abbeel, M.~I. Jordan, and P.~Moritz.
\newblock Trust region policy optimization.
\newblock In \emph{Proceedings of the 32nd International Conference on Machine
  Learning}, Lille, France, 2015. JMLR.org.

\bibitem[Schulman et~al.(2017)Schulman, Wolski, Dhariwal, Radford, and
  Klimov]{PPO:17}
J.~Schulman, F.~Wolski, P.~Dhariwal, A.~Radford, and O.~Klimov.
\newblock Proximal policy optimization algorithms.
\newblock \emph{CoRR}, 2017.

\bibitem[Sokota et~al.(2022)Sokota, Ryan~D'Orazio, Loizou, Lanctot, Mitliagkas,
  Brown, and Kroer]{MMD:22}
S.~Sokota, J.~Z.~K. Ryan~D'Orazio, N.~Loizou, M.~Lanctot, I.~Mitliagkas,
  N.~Brown, and C.~Kroer.
\newblock A unified approach to reinforcement learning, quantal response
  equilibria, and two-player zero-sum games, 2022.

\bibitem[Srinivasan et~al.(2018)Srinivasan, Lanctot, Zambaldi, P{\'{e}}rolat,
  Tuyls, Munos, and Bowling]{regret_pg:18}
S.~Srinivasan, M.~Lanctot, V.~F. Zambaldi, J.~P{\'{e}}rolat, K.~Tuyls,
  R.~Munos, and M.~Bowling.
\newblock Actor-critic policy optimization in partially observable multiagent
  environments.
\newblock In \emph{Advances in Neural Information Processing Systems 31},
  Montr{\'{e}}al, Canada, 2018.

\bibitem[Sutton and Barto(2018)]{RLbook:2018}
R.~S. Sutton and A.~G. Barto.
\newblock \emph{Reinforcement learning: {A}n introduction}.
\newblock {MIT} Press, 2018.

\bibitem[Syrgkanis et~al.(2015)Syrgkanis, Agarwal, Luo, and
  Schapire]{fast_reg_in_games:15}
V.~Syrgkanis, A.~Agarwal, H.~Luo, and R.~E. Schapire.
\newblock Fast convergence of regularized learning in games.
\newblock In \emph{Advances in Neural Information Processing Systems 28}, pages
  2989--2997, Montreal, Quebec, Canada, 2015.

\bibitem[Todorov et~al.(2012)Todorov, Erez, and Tassa]{MuJoCo}
E.~Todorov, T.~Erez, and Y.~Tassa.
\newblock {MuJoCo}: {A} physics engine for model-based control.
\newblock In \emph{2012 {IEEE/RSJ} International Conference on Intelligent
  Robots and Systems}, pages 5026--5033, Vilamoura, Algarve, Portugal, 2012.

\bibitem[Waugh and Bagnell(2015)]{connect_ol_oco}
K.~Waugh and J.~A. Bagnell.
\newblock A unified view of large-scale zero-sum equilibrium computation.
\newblock In \emph{Computer Poker and Imperfect Information, Papers from the
  2015 {AAAI} Workshop}, Austin, Texas, 2015. {AAAI} Press.

\bibitem[Wellman(2006)]{egta:06}
M.~P. Wellman.
\newblock Methods for empirical game-theoretic analysis.
\newblock In \emph{Proceedings, The Twenty-First National Conference on
  Artificial Intelligence and the Eighteenth Innovative Applications of
  Artificial Intelligence Conference}, pages 1552--1556, Boston, Massachusetts,
  2006. {AAAI} Press.

\bibitem[Yang et~al.(2019)Yang, Li, and Zhang]{RAC_19}
W.~Yang, X.~Li, and Z.~Zhang.
\newblock A regularized approach to sparse optimal policy in reinforcement
  learning.
\newblock In \emph{Advances in Neural Information Processing Systems 32},
  Vancouver, BC, Canada, 2019.

\bibitem[Zhang et~al.(2022)Zhang, Liu, Wang, Xiong, Li, and Bai]{PO4MG:22}
R.~Zhang, Q.~Liu, H.~Wang, C.~Xiong, N.~Li, and Y.~Bai.
\newblock Policy optimization for markov games: Unified framework and faster
  convergence.
\newblock \emph{CoRR}, abs/2206.02640, 2022.

\bibitem[Ziebart et~al.(2008)Ziebart, Maas, Bagnell, and Dey]{maxEntIRL:08}
B.~D. Ziebart, A.~L. Maas, J.~A. Bagnell, and A.~K. Dey.
\newblock Maximum entropy inverse reinforcement learning.
\newblock In \emph{Proceedings of the Twenty-Third {AAAI} Conference on
  Artificial Intelligence}, Chicago, Illinois, 2008. {AAAI} Press.

\bibitem[Zinkevich et~al.(2007)Zinkevich, Johanson, Bowling, and
  Piccione]{CFR:07}
M.~Zinkevich, M.~Johanson, M.~H. Bowling, and C.~Piccione.
\newblock Regret minimization in games with incomplete information.
\newblock In \emph{Advances in Neural Information Processing Systems 20}, pages
  1729--1736, Vancouver, British Columbia, Canada, 2007.

\end{thebibliography}

\setcounter{secnumdepth}{1}
\appendix
\onecolumn
\section{Recap of Notations}
In a compact space $A$, let $\int_{\mathrm{d}\mathbf{a}}f(\mathbf{a})=\int_{\mathbf{a}\in A}f(\mathbf{a})\mathrm{d}\mathbf{\mu(a)}$ for short, where $\mu$ is the Lebesgue measure and $\int {\mathrm{d}\mathbf{a}}=\text{vol}(A)$ is the volume of the compact space $A$. We will use $\mathrm{d}a$ for simplicity if there is no ambiguity. The inner product of two functions $f,g$ is defined as $\langle f,g \rangle=\int_{\mathrm{d}a}f(a)g(a)$. If $f,g$ are vectors, then $\langle f,g \rangle$ is the inner product of vectors. In the following, $\Vert f \Vert_p=(\int_{\mathrm{d}a}\vert f(a)\vert^p)^\frac{1}{p}$ is the $L_p$ norm of a function $f$. By default, $\Vert f\Vert$ is the $L_2$ norm. The (differential) entropy of a density function $p$ is $\mathcal{H}(p)=-\int_{a}p(a)\log p(a)\mathrm{d}a$ and Kullback-Leibler divergence $KL(p,q)=\int_a p(a)\log\frac{p(a)}{q(a)}\mathrm{d}a$.

\section{Calculus of Variation and Euler-Lagrange Equation}
We consider the following integral functional, where the argument $x$ can be a real number or a vector, i.e., $x\in R^n$,
$$F(y)=\int_{A}L(x,y(x), y^\prime(x))\mathrm{d}x.$$
For calculus of variation, the variation of $y(x)$ is often denoted as $\eta(x)$. The first variation of the functional $F$ is $\delta F(\eta,y)=\int \frac{\delta F}{\delta y(x)}\eta(x)\mathrm{d}x$, where $\frac{\delta F}{\delta y(x)}$ is the functional derivative. 
The above integral functional has functional derivative (Euler-Lagrange Equation) given by
$$\frac{\delta F}{\delta y(x)}=\frac{\partial L}{\partial y}-\sum_{i=1}^n\frac{\partial }{\partial x_i}\frac{\partial L}{\partial y_{x_i}}$$
By setting the Euler-Lagrange equation to zero, we can find an extrema of $F(y)$. Also, we can add a Lagrange multiplier to cope with the following form of constraints
$$G(y)=\int_{A}M(x,y(x), y^\prime(x))\mathrm{d}x=C.$$
We can solve the constrained problem by building a new functional in the same vein of the method of Lagrange multipliers:
$$H(y)=\int_{A}(L(x,y,y^\prime)-\lambda M(x,y,y^\prime))\mathrm{d}x.$$

For example, consider the following functional $L$ with the constraint that $g$ is a density function, i.e., $\int_{x\in A}g(x)\mathrm{d}x=1$ and $\forall x\in A, g(x)\ge 0$
$$
L(g)=\int_{x\in A} g(x)(-u(x)+\frac{\log g(x)}{\eta})\mathrm{d}x
$$
We can rewrite the above functional with constraint as follows:
\begin{equation}
F(g)=\int_{x\in A} g(x)(-u(x)+\frac{\log g(x)}{\eta}-\lambda)\mathrm{d}x \label{exam_Ifunctional}
\end{equation}
Noticed this functional does not depend on any partial derivative of $g(x)$, the functional derivative is:
\begin{equation}
\frac{\delta F}{\delta g}=(-u(x)+\frac{\log g(x)+1}{\eta}-\lambda). \label{exam_for_CV}
\end{equation}

\section{Finding the Extrema of a Functional by the Euler-Lagrange Equation}
The necessary condition for $y$ being a minimum of a functional is that the Euler-Lagrange equation is 0 and the second variation $$\delta^2 F(\eta,y)=\int [\frac{\partial^2 F}{\partial y^2}\eta^2(x)+2\frac{\partial^2 F}{\partial y\partial y^\prime}\eta(x)\eta^\prime(x)+\frac{\partial^2 F}{\partial(y^\prime)^2}(\eta^\prime(x))^2]\mathrm{d}x\ge0, \forall \eta(x).$$
For instance, let the Euler-Lagrange equation of Eq.\eqref{exam_Ifunctional} be 0, we can get 
$$g(x)=\exp(\eta\cdot(u(x)+\lambda)) \text{ and } \frac{\partial^2 F}{\partial g^2}=\frac{1}{\eta g(x)}>0.$$
Note that Eq.\eqref{exam_Ifunctional} does not depend on $g^\prime(x)$, thus the second variation $\delta^2 F(\eta,g)=\int \frac{\partial^2 F}{\partial g^2}\eta^2(x)\mathrm{d}x=\int \frac{\eta^2(x)}{\eta^2 g^2(x)}\mathrm{d}x\ge 0$. So $g(x)=\exp(\eta\cdot(u(x)+\lambda))$ is a (local) minimum. From the constraint that $\int g(x)\mathrm{d}x=1$, we can remove $\lambda$ and obtain $$g(x)=\frac{\exp(\eta \cdot u(x))}{\int \exp(\eta\cdot u(x))\mathrm{d}x}.$$
Similarly, it can be easily verified that $g(x)=\frac{\exp(\eta \cdot u(x))}{\int \exp(\eta\cdot u(x))\mathrm{d}x}$ is a maximum of 
$$L(g)=\int_{x\in A}g(x)(u(x)-\frac{\log g(x)}{\eta})\mathrm{d}x$$ where $g(x)$ is a density function.


\section{Proof Details}
\label{detailed_proof}

\subsection{Proof of Theorem \ref{vanish_kl}}

\begin{lemma} \label{kl_l2_ineq}
Given two density functions $p,q$ on a compact set $\mathcal{X}$, where we assume that $0<a\le p(x)\le b$ and $0<a\le q(x)\le b$, then 
\begin{equation}
KL(p,q)\ge \frac{1}{2b}\cdot \Vert p(x)-q(x)\Vert^2_2=\frac{1}{2b}\int_{x} (p(x)-q(x))^2\mathrm{d}x
\end{equation}
\end{lemma}
\begin{proof}
Let $h(t)=2(2+t)(t\log t-t+1)-3(t-1)^2,t>0.$ We have 
$$h^\prime(t)=2(t\log t-t+1)+2(2+t)\log t-6(t-1)=4\big((1+t)\log t-2(t-1)\big),h^{\prime\prime}=4(\frac{1}{t}+\log t-1).$$
Notice that $h^{\prime\prime}(1)=0$ which is the minimum of $h^{\prime\prime}$, thus $h(t)$ is convex. Also notice that $h^\prime(1)=0$, so $h(t)\ge h(1)=0$. We have that
\begin{equation}
t\log t-t+1\ge\frac{3(t-1)^2}{2(2+t)}. \label{aux_ineq}    
\end{equation}
Taking $t(x)=\frac{p(x)}{q(x)}>0$ into Eq.\eqref{aux_ineq}, we get
\begin{equation}
    \frac{p(x)}{q(x)}\log\frac{p(x)}{q(x)}-\frac{p(x)}{q(x)}+1\ge \frac{3(\frac{p(x)}{q(x)}-1)^2}{2(2+\frac{p(x)}{q(x)})}.
\end{equation}
Multiplying by $q(x)$ and taking integral on both sides, we have
\begin{equation}
    \int_{x\in\mathcal{X}} \big(p(x)\log\frac{p(x)}{q(x)}-p(x)+q(x)\big)\mathrm{d}x=KL(p,q)\ge \int_{x\in\mathcal{X}} \frac{3(p(x)-q(x))^2}{2(p(x)+2q(x))}\mathrm{d}x.
\end{equation}
By the assumption that $p(x)$ and $q(x)$ are bounded, we have $p(x)+2q(x)\le 3b$. So we get
\begin{equation}
    KL(p,q)\ge \frac{1}{2b}\cdot\int_{x\in\mathcal{X}} (p(x)-q(x))^2\mathrm{d}x.
\end{equation}
\end{proof}



Then we proceed the proof of Theorem~\ref{vanish_kl}.
\begin{proof}
[\textbf{Proof} of Theorem \ref{vanish_kl}]
When we use the $\max\min$ objective, we often assume the game is zero-sum, i.e., $Q^1(a_1,a_2)=Q(a_1,a_2)=-Q^2(a_1,a_2)$. Equivalently, both players maximize their own value function.
Let $Q_t(a)=Q^{\pi_{1,t},\pi_{2,t}}(a)=(Q^1_t(a_1),-Q^2_t(a_2))$ for short, where
$$
Q^1_t(a_1)=\int \pi_{2,t}(a_2)Q(a_1,a_2)\mathrm{d}a_2 \text{ and } Q^2_t(a_2)=\int \pi_{1,t}(a_1)Q(a_1,a_2)\mathrm{d}a_1.
$$
Thus, we consider the joint policy $\pi_t=(\pi_{1,t},\pi_{2,t})$ optimization in the main text. That is, 
\begin{align}
&\langle \pi(a), Q_t(a)\rangle=\int \pi_1(a_1)Q^1_t(a_1)\mathrm{d}a_1 +\int \pi_2(a_2)(-Q^2_t(a_2))\mathrm{d}a_2=\int \pi_1(a_1)Q^1_t(a_1)\mathrm{d}a_1 -\int \pi_2(a_2)Q^2_t(a_2)\mathrm{d}a_2, \nonumber\\
&=\int \pi_1(a_1)\int \pi_{2,t}(a_2)Q(a_1,a_2)\mathrm{d}a_1\mathrm{d}a_2 -\int \pi_2(a_2)\int \pi_{1,t}(a_1)Q(a_1,a_2)\mathrm{d}a_1\mathrm{d}a_2\\
&\int_{a}\pi(a)\log\pi(a)\mathrm{d}a=\int\pi_1(a_1)\log\pi_1(a_1)\mathrm{d}a_1+\int\pi_2(a)\log\pi_2(a)\mathrm{d}a_2,\\
&KL(\pi,\pi_t)=KL(\pi_1,\pi_{1,t})+KL(\pi_2,\pi_{2,t}).
\end{align}
Denote 
\begin{align}
G(\pi)&=\langle\pi(a),Q(a)\rangle-\alpha\int\pi(a)\log\pi(a)\mathrm{d}a,\\
F(\pi)&=\langle\pi(a),Q_t(a)\rangle-\alpha{\int_{a}\pi(a)\log\pi(a)\mathrm{d}a}-\frac{1}{\eta}\int_a\pi(a)\log\frac{\pi(a)}{\pi_{t}(a)}\mathrm{d}a.
\end{align}
The \emph{soft} optimal joint policy 
\begin{align}
\pi^\star=(\pi^{\star}_1,\pi^{\star}_2)=\arg\max_{\pi^1}\arg\min_{\pi^2}G(\pi^1,\pi^2).
\end{align}
Recall that $\pi_{t+1}=\arg\max_{\pi}F_t(\pi).$
By the first-order optimality condition, we have
\begin{align}
\forall\pi,&-\nabla F(\pi_{t+1})(\pi-\pi_{t+1})\ge0\Leftrightarrow\\
&\langle Q_t(a)-\alpha\cdot\log\pi_{t+1}(a)-\frac{1}{\eta}(\log\frac{\pi_{t+1}(a)}{\pi_t(a)}+1),\pi-\pi_{t+1}\rangle\le0\Leftrightarrow\\
&\langle Q_t(a)-\alpha\cdot\log\pi_{t+1}(a),\pi-\pi_{t+1}\rangle\le \frac{1}{\eta}\langle\log\frac{\pi_{t+1}(a)}{\pi_t(a)},\pi-\pi_{t+1}\rangle
\end{align}
We can also decompose the KL as 
\begin{equation} \label{kl_decompose}
    KL(\bar{z},y)-KL(z,y)+KL(z,\bar{z})=\langle\log\bar{z}-\log y, \bar{z}-z\rangle.
\end{equation}
By Eq.\eqref{kl_decompose}, we have
\begin{eqnarray}
\langle Q_t(a)-\alpha\cdot\log\pi_{t+1}(a),\pi^\star-\pi_{t+1}\rangle\le \frac{1}{\eta}(KL(\pi^\star,\pi_{t})-KL(\pi_{t+1},\pi_{t})-KL(\pi^\star,\pi_{t+1}))   . 
\end{eqnarray}
Rearrange the above, we can get 
\begin{align}
KL(\pi^*,\pi_{t+1}) &\le KL(\pi^*,\pi_t)-KL(\pi_{t+1},\pi_t)-\eta\langle Q_t(a)-\alpha\cdot\log\pi_{t+1}(a),\pi^\star-\pi_{t+1}\rangle\Leftrightarrow\\
KL(\pi^*,\pi_{t+1})&\le KL(\pi^*,\pi_t)-KL(\pi_{t+1},\pi_t)+\eta\langle Q_t(a)-\alpha\cdot\log\pi_{t+1}(a),\pi_{t+1}-\pi^\star\rangle\\
&=KL(\pi^*,\pi_t)-KL(\pi_{t+1},\pi_t)+\eta\langle Q_t(a)-Q_{t+1}(a),\pi_{t+1}-\pi^\star \rangle+\eta\langle Q_{t+1}(a)-\alpha\log\pi_{t+1}(a), \pi_{t+1}-\pi^\star\rangle \label{recur_phase1_last}
\end{align}
Taking the last term $\eta\langle Q_{t+1}(a)-\alpha\log\pi_{t+1}(a), \pi_{t+1}-\pi^\star\rangle$ and denoting $\tilde{Q}^\star=(\langle Q,\pi^{\star}_2\rangle,-\langle Q, \pi^\star_1\rangle)$, we have 
\begin{align}
&\eta\langle Q_{t+1}(a)-\alpha\log\pi_{t+1}(a), \pi_{t+1}-\pi^\star\rangle\\
&=\eta\underbrace{\langle Q_{t+1}-\tilde{Q}^\star,\pi_{t+1}-\pi^\star\rangle}_{=0, \text{ by definition}} + \eta\alpha\langle\log\pi^\star-\log\pi_{t+1},\pi_{t+1}-\pi^\star\rangle + \underbrace{ \eta\langle \tilde{Q}^\star-\alpha\log\pi^\star,\pi_{t+1}-\pi^\star\rangle}_{\le 0, \text{by the optimality of }\pi^\star}\\
&\le \eta\alpha\langle\log\pi^\star-\log\pi_{t+1},\pi_{t+1}-\pi^\star\rangle\\
&=-\eta\alpha (KL(\pi^\star,\pi_{t+1})+KL(\pi_{t+1},\pi^\star)).
\end{align}
Plugging back to Eq.\eqref{recur_phase1_last}, we get
\begin{align}
KL(\pi^*,\pi_{t+1})&\le KL(\pi^*,\pi_t)-KL(\pi_{t+1},\pi_t)+\eta\langle Q_t(a)-Q_{t+1}(a),\pi_{t+1}-\pi^\star \rangle-\eta\alpha
(KL(\pi^\star,\pi_{t+1})+KL(\pi_{t+1},\pi^\star)).
\end{align}
By Cauchy-Schwarz inequality for functions, $\langle Q_t(a)-Q_{t+1}(a),\pi_{t+1}-\pi^\star \rangle\le \Vert Q_t-Q_{t+1}\Vert_2\Vert\pi_{t+1}-\pi^\star\Vert_2$, we have
\begin{align}
KL(\pi^*,\pi_{t+1})&\le KL(\pi^\star,\pi_t)-KL(\pi_{t+1},\pi_t)+\eta\Vert Q_t-Q_{t+1}\Vert_2\Vert\pi_{t+1}-\pi^\star\Vert_2-\eta\alpha (KL(\pi^\star,\pi_{t+1})+KL(\pi_{t+1},\pi^\star))\\
&\le KL(\pi^*,\pi_t)-KL(\pi_{t+1},\pi_t)+\eta\underbrace{L\Vert \pi_t-\pi_{t+1}\Vert_2}_{L\le\frac{R_{\max}}{1-\gamma}}\Vert\pi_{t+1}-\pi^\star\Vert_2 -\eta\alpha (KL(\pi^\star,\pi_{t+1})+KL(\pi_{t+1},\pi^\star))\label{before_elem_ineq}\\
&\le KL(\pi^*,\pi_t)-KL(\pi_{t+1},\pi_t)+\frac{\eta^{1/2}\Vert \pi_t-\pi_{t+1}\Vert_2^2}{2}+ \frac{\eta^{3/2} L^2\Vert\pi_{t+1}-\pi^\star\Vert_2^2}{2}\nonumber\\
&\quad-\eta\alpha (KL(\pi^\star,\pi_{t+1})+KL(\pi_{t+1},\pi^\star)) \label{after_elem_ineq}\\
&\overset{\text{by Lemma \ref{kl_l2_ineq}}}{\le} KL(\pi^*,\pi_t)-KL(\pi_{t+1},\pi_t)+\eta^{1/2}b\cdot KL(\pi_{t+1},\pi_t)+ \eta^{3/2}L^2 b\cdot KL(\pi_{t+1},\pi^\star)\nonumber\\ 
&\quad -\eta\alpha (KL(\pi^\star,\pi_{t+1})+KL(\pi_{t+1},\pi^\star))
\end{align}
Ineq.\eqref{after_elem_ineq} is obtained by the elementary inequality $2a\cdot b\le a^2+b^2$. By letting $\eta\le\min\{\frac{1}{b^2},\frac{\alpha^2}{b^2L^4}\}$, 
\begin{equation}
KL(\pi^*,\pi_{t+1})\le KL(\pi^\star,\pi_t)-\eta\alpha KL(\pi^\star,\pi_{t+1}).
\end{equation}
Repeat the above recursion, we have 
\begin{equation}
KL(\pi^\star,\pi_{t})\le (\frac{1}{1+\eta\alpha})^t KL(\pi^\star,\pi_0).
\end{equation}
\end{proof}

\section{Ablation Study}
\label{abaltion_study}
New hyper-parameters specific to PORL are $\alpha,\eta$ and the policy replacement interval $T$. According to Eq.\eqref{qre_obj}, $\alpha$ should be small to guarantee that the \emph{soft} optimal solution is close to the original optimal solution. We conduct ablation studies on the non-stationary HalfCheeta-v2 with 3 seeds, mainly varying $\eta$ and $T$. We keep the default $\alpha=0.2$ in non-stationary tasks. We first fix $\eta=10$, i.e., $1/\eta=0.1$ and vary $T$ in $\{10,100,500,1000,2000\}$. The final performances of the last policy are reported in Table~\ref{tab:ablation_t}. This implies that a large policy update interval may hinder the policy improvement, while smaller intervals can help achieve higher final performance. Although the best final performance becomes better as $T$ becomes smaller in Table~\ref{tab:ablation_t}, we find $T=1000$ works well for stationary and non-stationary tasks, and $T=10$ works better for adversarial training and competitive games. Intuitively, when the learning environments change more frequently, a smaller $T$ is more suitable.

\begin{table}[H]
    \centering
    \caption{Ablation studies on gradient steps for each iteration (fix $\eta=10$).}
\begin{tabular}{l|r@{~$\pm$~}lr@{~$\pm$~}lr@{~$\pm$~}lr@{~$\pm$~}lr@{~$\pm$~}l}\toprule
& \multicolumn{2}{c}{$T=2000$} & \multicolumn{2}{c}{$T=1000$} & \multicolumn{2}{c}{$T=500$} & \multicolumn{2}{c}{$T=100$} & \multicolumn{2}{c}{$T=10$}\\\midrule
HalfCheetah-v2 & $3187.9$& $1082.91$ & $3901.09$& $448.84$ & $4207.02$& $136.25$ & $4281.62$& $101.01$ & $ \mathbf{4545.25} $ & $ \mathbf{237.32} $\\
\bottomrule
    \end{tabular}
    \label{tab:ablation_t}
\end{table}

From theorem \ref{vanish_kl}, $\eta$ can be rather small. However, as the coefficient in Eq.\eqref{param_obj} is $1/\eta$, a very small $\eta$ will raise instability when we use neural network models. We then fix $T=1000$ and vary $\eta$ in $\{10,5,2,1,0.1\}$. The results are listed in Table~\ref{tab:ablation_eta}. We find PORL is not very sensitive to the choice of $\eta$ if $1/\eta$ is not too large. However, if $1/\eta$ becomes large, the policy can fail to converge and perform very badly.
\begin{table}[H]
    \centering
    \caption{Ablation studies on $\eta$ (fix $T=1000$).}
\begin{tabular}{l|r@{~$\pm$~}lr@{~$\pm$~}lr@{~$\pm$~}lr@{~$\pm$~}lr@{~$\pm$~}l}\toprule
& \multicolumn{2}{c}{$\eta=10.0$} & \multicolumn{2}{c}{$\eta=5.0$} & \multicolumn{2}{c}{$\eta=2.0$} & \multicolumn{2}{c}{$\eta=1.0$} & \multicolumn{2}{c}{$\eta=0.1$}\\\midrule
HalfCheetah-v2 & $3901.09$& $448.84$ & $ \mathbf{4759.63} $ & $ \mathbf{134.32} $ & $4243.37$& $91.41$ & $4465.8$& $150.72$ & $-572.36$& $9.09$\\
\bottomrule
\end{tabular}
    \label{tab:ablation_eta}
\end{table}

We also verify that a non-zero $\alpha$ is necessary for the last-iterative convergence. We simply set $\alpha=0$ in the two MPE competitive games and plot the returns of the learning policy in Figure~\ref{fig:mpe_PORL_abalaiton}. As these two games are not rigorously zero-sum, the performance gap of a joint policy will reduce but not necessarily vanish. As shown in Figure~\ref{fig:mpe_PORL_abalaiton}, the newest joint policy can fluctuate when $\alpha=0$, while for $\alpha=0.01$, the performance gap stably reduces.
\begin{figure}[H]
    \centering
    \includegraphics[width=0.9\textwidth]{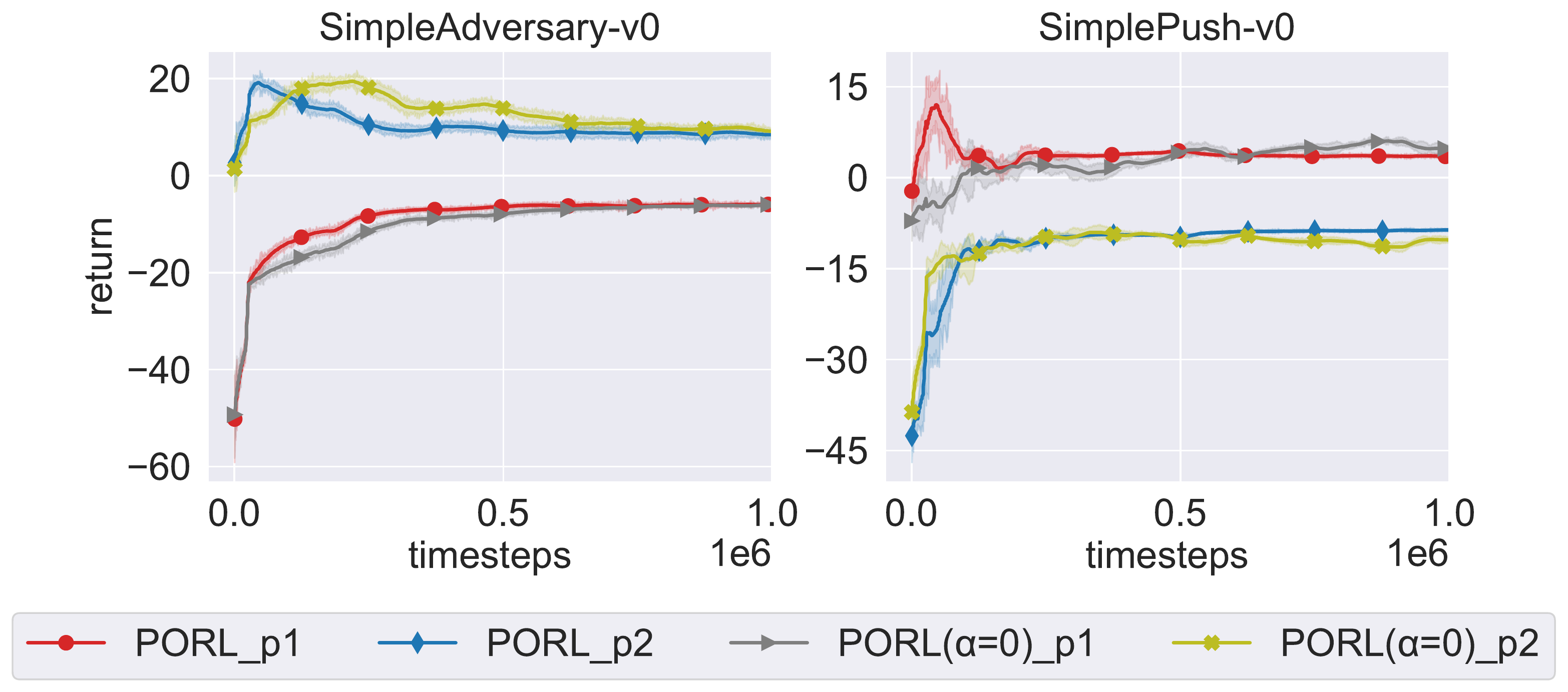}
    \caption{Return curves in $2$ multi-agent particle environment competitive tasks. The returns for player 1 (the adversary) and 2 (the main agent) are plotted respectively. The curves are shaded with standard errors. If an algorithm can converge, the performance between the two players necessarily becomes smaller as the training proceeds.}
    \label{fig:mpe_PORL_abalaiton}
\end{figure}

We perform an extra ablation on the anticipatory parameter $\eta$ in NFSP. As we observe from \citep{BCQ}, if the behavior policy and the learning policy are too disjoint, the off-policy algorithm cannot easily obtain an effective best response policy from the replay buffer. When $\eta$ in NFSP is small, the learning of BR is less effective, while when $\eta$ is large, the average policy converges more slowly. As shown in Figure \ref{fig:mpe_NFSP_ret}, a small $\eta$ does fail to converge and a large $\eta$ converges slowly.
\begin{figure}[H]
    \centering
    \includegraphics[width=0.9\textwidth]{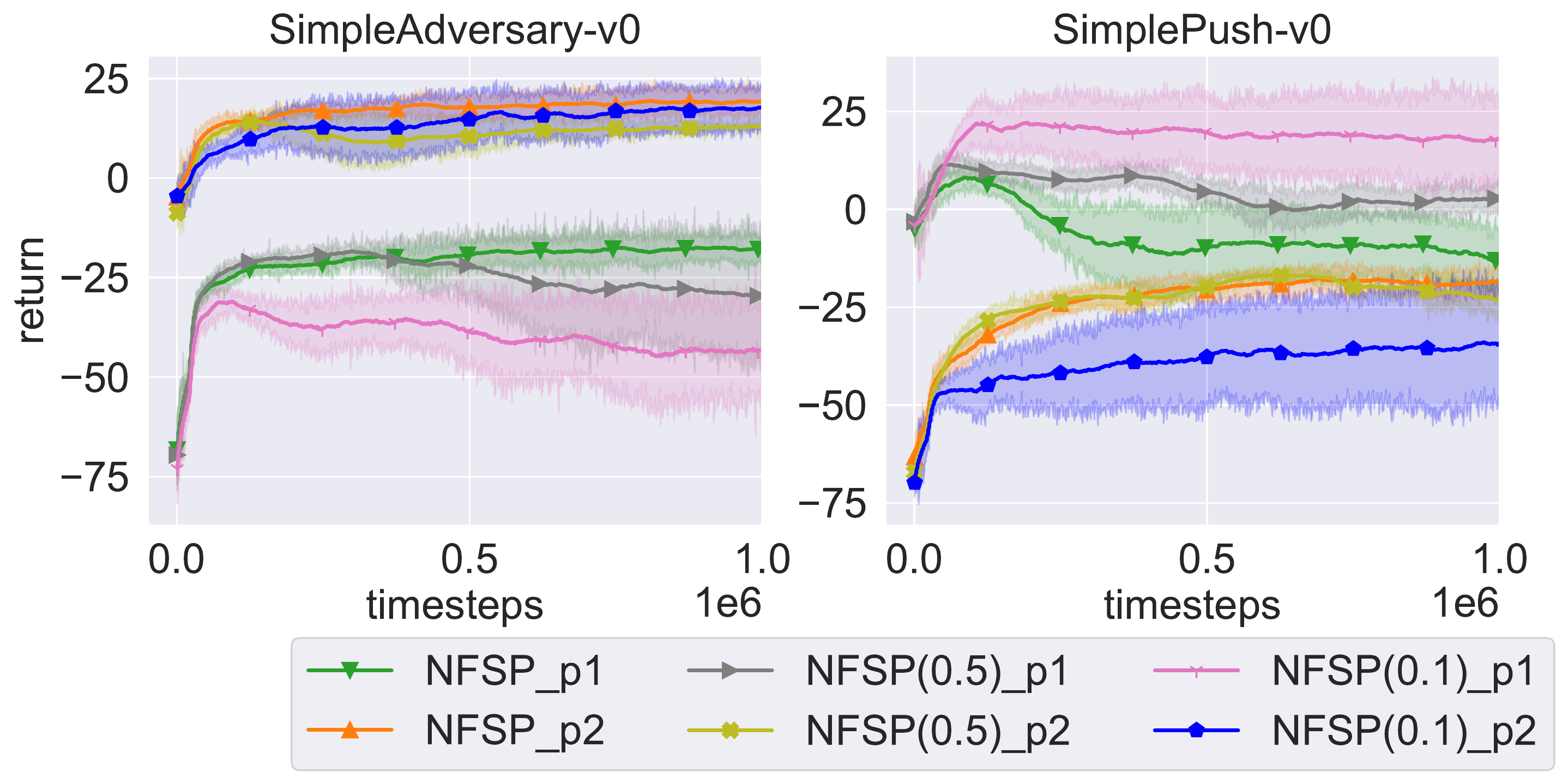}
    \caption{Return curves in $2$ multi-agent particle environment competitive tasks. The returns for player 1 (the adversary) and 2 (the main agent) are plotted respectively. The curves are shaded with standard errors.}
    \label{fig:mpe_NFSP_ret}
\end{figure}

\section{Hyper-parameters and Implementation Details}
\label{param_detail}
For SAC, PORL, and NFSP's best response oracle, which is also SAC, we all use double $Q$ and target $Q$ nets. For Line 5 and 8 in algorithm \ref{orm}, we only perform one gradient update, i.e., $K_1=K_2=1$ in all the experiments. Then we replace the $Q$ function and previous round policy (Line 10 in algorithm \ref{orm}) every $T$ gradient updates. All the value networks and policy networks share the same parameters in the same task. The network architecture and learning rates are referenced from openAI baselines for MuJoCo and MPE \citep{mpe17}. The shared parameters are listed in Table \ref{table_supp_share_hyper_parameters} and \ref{table_supp_mpe_share_hyper_parameters}. The parameters used in the final experiments are list in Table \ref{table_supp_porl_hyper_parameters}, \ref{table_supp_sac_hyper_parameters}, and \ref{table_supp_nfsp_hyper_parameters}. We also perform a simple grid search for PORL in adversarial training and MPE, and for NFSP in MPE. The search space are listed in Table \ref{table_supp_porl_grid_search_range} and \ref{table_supp_nfsp_grid_search_range}.

\begin{table}[H]
    \centering
    \caption{Shared hyper-parameters of PORL and SAC on MuJoCo domains.}
    \label{table_supp_share_hyper_parameters}
    \begin{tabular}{lr}\toprule
\textbf{Attributes} & \textbf{Value}\\\midrule
policy learning rate&  $3e-4$\\
value learning rate&  $1e-3$\\
discount factor $\gamma$ & $0.99$ \\
batch size & $256$ \\
target value updating factor $\tau$ & $0.995$ \\
\# policy gradient steps per sample & 1 \\
\# policy hidden units & $[256, 64]$\\
activation function of policy hidden layers & leaky\_relu \\
\# value gradient steps per sample & 1 \\
\# value hidden units & $[256, 64]$\\
activation function of policy hidden layers & leaky\_relu \\
\# replay buffer & $1e6$\\
\bottomrule
\end{tabular}
\end{table}

\begin{table}[H]
    \centering
    \caption{Shared hyper-parameters of PORL, SAC and NFSP on MPE domains.}
    \label{table_supp_mpe_share_hyper_parameters}
    \begin{tabular}{lr}\toprule
\textbf{Attributes} & \textbf{Value}\\\midrule
policy learning rate&  $1e-3$\\
value learning rate&  $1e-3$\\
discount factor $\gamma$ & $0.95$ \\
batch size & $256$ \\
target value updating factor $\tau$ & $0.995$ \\
\# policy gradient steps per sample & 1 \\
\# policy hidden units & $[64, 64]$\\
activation function of policy hidden layers & leaky\_relu \\
\# value gradient steps per sample & 1 \\
\# value hidden units & $[64, 64]$\\
activation function of policy hidden layers & leaky\_relu \\
\# replay buffer & $2e5$\\
\bottomrule
\end{tabular}
\end{table}

\begin{table}[H]
    \centering
    \caption{Hyper-parameters of PORL used in the experiments.}
    \label{table_supp_porl_hyper_parameters}
    \begin{tabular}{lr}\toprule
\textbf{Attributes} & \textbf{Value}\\\midrule
$\eta$ &  $10.0$\\
value of $T$ & $\begin{cases}  10 \text{ (MPE, GAIL)}\\  1,000 \text{ ((non-)stationary MuJoCo)} \end{cases}$ \\
$\alpha$ & $\begin{cases} 0.01 \text{ (MPE, GAIL)} \\ 0.2 \text{ ((non-)stationary MuJoCo)} \end{cases}$ \\                   
\bottomrule
\end{tabular}
\end{table}

\begin{table}[H]
    \centering
    \caption{Hyper-parameters of SAC used in the experiments.}
    \label{table_supp_sac_hyper_parameters}
    \begin{tabular}{lr}\toprule
\textbf{Attributes} & \textbf{Value}\\\midrule
target entropy for auto-tuning $\alpha$& \#Dim($\mathcal{A}$)$\times 1.5$\\
\bottomrule
\end{tabular}
\end{table}

\begin{table}[H]
    \centering
    \caption{Hyper-parameters of NFSP used in the experiments.}
    \label{table_supp_nfsp_hyper_parameters}
    \begin{tabular}{lr}\toprule
\textbf{Attributes} & \textbf{Value}\\\midrule
target entropy for auto-tuning $\alpha$& \#Dim($\mathcal{A}$)$\times 1.5$\\
anticipatory parameter $\eta$ & $0.2$\\
\# average policy buffer & $2e6$\\
learning rate for average policy & $1e-3$\\
\bottomrule
\end{tabular}
\end{table}

\begin{table}[H]
    \centering
    \caption{Grid search range for $\eta$, $T$, and $\alpha$ of PORL on adversarial training and MPE.}
    \label{table_supp_porl_grid_search_range}
    \begin{tabular}{lr}\toprule
\textbf{Attributes} & \textbf{Value}\\\midrule
$\alpha$ & $\{0.01, 0.02, 0.05,0.1\}$ \\
$\eta$ & $\{10.0, 5.0, 2.5\}$ \\
$T$ & $\{10, 100\}$ \\
\bottomrule
\end{tabular}
\end{table}

\begin{table}[H]
    \centering
    \caption{Grid search range for $\eta$ of NFSP on MPE.}
    \label{table_supp_nfsp_grid_search_range}
    \begin{tabular}{lr}\toprule
\textbf{Attributes} & \textbf{Value}\\\midrule
$\eta$ & $\{0.1, 0.2, 0.5\}$ \\
\bottomrule
\end{tabular}
\end{table}

\end{document}